\documentclass[12pt,onecolumn]{IEEEtran}
\usepackage[OT1]{fontenc}
\usepackage{amsthm,amsmath,amssymb}
\usepackage[numbers]{natbib}
\usepackage{paralist}
\newcommand{\subparagraph}{}
\usepackage[compact]{titlesec}
\usepackage{setspace}

\usepackage{amsthm,amsmath,natbib}
\usepackage{mathrsfs}
\usepackage{algorithm}
\usepackage{algorithmic}
\usepackage{amssymb}
\usepackage{multirow}
\usepackage[english]{babel}
\usepackage{graphicx}
\usepackage{subfig}
\usepackage{url}

\newcommand{\deliver}{{\bf deliver}}
\newcommand{\given}{{\bf given}}

\newtheorem{theorem}{Theorem}
\newtheorem{lemma}{Lemma}

\newcommand{\sgn}{{\rm sgn}}
\newcommand{\bsb}{\boldsymbol}
\newcommand{\bsbY}{{\boldsymbol{Y}}}
\newcommand{\bsbX}{{\boldsymbol{X}}}
\newcommand{\bsbx}{{\boldsymbol{x}}}
\newcommand{\bsbc}{{\boldsymbol{c}}}
\newcommand{\bsbf}{{\boldsymbol{f}}}
\newcommand{\bsbl}{{\boldsymbol{l}}}
\newcommand{\bsby}{{\boldsymbol{y}}}

\newcommand{\bsbw}{{\boldsymbol{w}}}
\newcommand{\bsbb}{{\boldsymbol{\beta}}}
\newcommand{\bsbg}{{\boldsymbol{\gamma}}}

\newcommand{\bsbLamb}{{\boldsymbol{\Lambda}}}

\newcommand{\bsbH}{{\boldsymbol{H}}}
\newcommand{\bsbG}{{\boldsymbol{G}}}
\newcommand{\bsbI}{{\boldsymbol{I}}}
\newcommand{\bsbL}{{\boldsymbol{L}}}

\newcommand{\bsbZ}{{\boldsymbol{Z}}}
\newcommand{\bsbz}{{\boldsymbol{z}}}
\newcommand{\bsbR}{{\boldsymbol{R}}}
\newcommand{\bsbQ}{{\boldsymbol{Q}}}

\newcommand{\bsbSig}{{\boldsymbol{\Sigma}}}
\newcommand{\bsbxi}{{\boldsymbol{\xi}}}
\newcommand{\bsbXi}{{\boldsymbol{\Xi}}}
\newcommand{\bsbD}{{\boldsymbol{D}}}
\newcommand{\bsbW}{{\boldsymbol{W}}}

\newcommand{\bsbB}{{\boldsymbol{B}}}
\newcommand{\bsbP}{{\boldsymbol{P}}}
\newcommand{\bsbU}{{\boldsymbol{U}}}

\newcommand{\bsbK}{{\boldsymbol{K}}}
\newcommand{\bsba}{{\boldsymbol{\alpha}}}
\newcommand{\bsbA}{{\boldsymbol{A}}}

\newcommand{\bsbC}{{\boldsymbol{C}}}

\newcommand{\bsbpi}{{\boldsymbol{\pi}}}

\newcommand{\bsbu}{{\boldsymbol{u}}}

\newcommand{\bsbmu}{{\boldsymbol{\mu}}}
\newcommand{\rd}{\,\mathrm{d}}
\newcommand{\tran}{{\mathsf{T}}}

\newcommand{\sss}{\textbf{S}$^3$}

\begin{document}

\title{Learning  Topology and Dynamics of Large Recurrent Neural Networks}
\author{Yiyuan She, Yuejia He,  and Dapeng Wu \thanks{Yiyuan She is with Department of Statistics, Florida State University, Tallahassee, FL 32306. Yuejia He and Dapeng Wu are with Department of Electrical and Computer Engineering,   University of Florida, Gainesville, FL 32611.
                  }}

\maketitle

%\title{Learning the Topology and Dynamics of Large Recurrent Neural Networks}
%\author{Yiyuan She, Yuejia He and Dapeng Wu}
%\date{}
%\maketitle
\begin{abstract}
Large-scale recurrent networks have drawn increasing attention recently  because of  their capabilities in  modeling  a large variety of real-world phenomena and physical mechanisms.
This paper studies  how to   identify all authentic connections  and estimate  system parameters of a  recurrent network, given a sequence of   node observations. This task becomes extremely challenging in modern   network applications,  because  the available observations  are usually very noisy and limited, and  the associated dynamical system is strongly nonlinear. By  formulating the problem as multivariate sparse sigmoidal regression, we   develop simple-to-implement  network learning algorithms,  with rigorous  convergence guarantee in theory, for a variety of sparsity-promoting penalty forms. A quantile variant of  progressive recurrent network screening is proposed for efficient computation and  allows for  direct cardinality control  of network topology in estimation.  Moreover, we investigate recurrent network stability conditions  in {Lyapunov}'s sense, and integrate such stability constraints into  sparse network learning. Experiments show excellent performance of the proposed algorithms in network topology identification and forecasting.
\end{abstract}

\begin{keywords}
Recurrent networks, topology learning, shrinkage estimation, variable selection, dynamical systems, Lyapunov  stability.
\end{keywords}

\section{Introduction}
\label{sec:intro}

There has been an increasing interest in identifying network
dynamics and topologies in the emerging scientific discipline of
network science. In a dynamical network,
the evolution of a node is controlled not only by itself, but also
by other nodes. For example, in  gene regulatory networks
\citep{Faith07}, the expression levels of genes influence each
other, following some dynamic rules, such that the genes
are connected together to form a dynamical system. If the topology and evolution
rules of the network are known, we can analyze the regulation
between genes or detect unusual behaviors to help diagnose and cure
genetic diseases. Similarly, the modeling and estimation of
dynamical networks are of great importance in various domains
including stock market, brain network
 and social network \citep{mills2008econometric,Bullmore09,Hanneke10}. To
accurately identify the topology and dynamics underlying those
networks, scientists are devoted to developing appropriate mathematical
models and corresponding estimation methods.
%Large-scale dynamical networks have drawn increasing attention from researchers  recently, because of  their capabilities in  modeling  a large variety of real-world phenomena and physical mechanisms.
%The state of a recurrent network evolves through node interactions.  Therefore, topology learning and parameter estimation of the associated dynamical system are   meaningful and crucial to reveal  the network structure.

 In the literature, linear dynamical models are    commonly used. For example, the human brain connectivity network \cite{roebroeck2005mapping} can be  characterized by a set of linear differential equations, where  the  rate of change of activation/observation  of any  node is  a weighted sum of  the activations/observations of its neighbors:
$
{\rd x_i}/{\rd t} = \sum_{j\neq i} \alpha_{ij} x_j - d_{i} x_i , 1\leq i\leq n.
$
Here $\alpha_{ij}$ provide the connection weights and $d_i$ is the decay rate.
 %The model parameters can be learned using observations from, say, MRI or EEG experiments.
Nevertheless, a lot of complex dynamical networks clearly demonstrate   \textit{nonlinear} relationships between the nodes. For instance,    the strength of influence is unbounded in the previous simple linear combination, but the so-called   ``\textbf{saturation}'' effect widely exists in physical/biological systems (neurons, genes, and stocks)---the external influence  on a node, no matter how strong the total input activation is,  cannot go beyond  a certain threshold.  To capture the mechanism,  nonlinearity must be introduced into the network system:
$
{\rd x_i}/{\rd t} = l_{i} \pi(\sum_{j\neq i} \alpha_{ij} x_j+u_i) - d_{i} x_i + c_i$,
where $\pi$ denotes  a nonlinear activation function  typically taken to be  the sigmoidal  function
$
\pi(\theta) =  {1}/({1+e^{-\theta}})
$. It has a proper shape to resemble   many real-world mechanisms and behaviors.

The model description  is associated with  a   \textit{continuous-time recurrent neural network}.
 The existing feedback loops allow the network to exhibit interesting dynamic temporal behaviors  to capture many kinds of relationships. It is also biologically realistic  in   say modeling the effect of an input spike train. Recurrent networks have been  successfully applied to a wide range of problems in bioinformatics,  financial market forecast, electric circuits,  computer vision, and robotics; see, e.g., \cite{beer1997dynamics,gallacher2000continuous,robotics,xu2007inference,Vu2007} among many others. %
%This model originated from  modeling the effects on a neuron of the incoming spike train and have been successfully applied to a variety of research areas, such as bioinformatics, electric circuits and evolutionary robotics \cite{xu2007inference,beer1997dynamics,gallacher2000continuous}. A key feature of a recurrent network is the feedback loops
%in which the activation can flow round, allowing the network to exhibit dynamic temporal behaviors.

In practical applications, it is often necessary to include      noise contamination:
$
\rd x_i = (l_i \pi( \sum_{j\neq i} \alpha_{ij}  x_j + u_i) - d_i x_i + c_i)\rd t + \sigma \rd \mathcal{B}_t,    1\leq i\leq n,
$
where  $\mathcal{B}_t$ stands for an $n$-dimensional Brownian motion.
% Such a  model has been applied to many biological and chemical systems, such as neural networks and gene networks \cite{marnellos1998gene, vohradsky2001neural}. Studying it benefits us in understanding those dynamical systems and processes of interest.
Among the very many unknown parameters,  $\alpha_{ij}$ might be the most important: the zero-nonzero pattern of $\alpha_{ij}$ indicates if  there exists a (direct) connection from node $j$ to node $i$.  Collecting all such connections results in a {directed} graph to describe the node  interaction structure.

A fundamental question naturally arises: \textit{Given a sequence of  node observations   (possibly at very few time points), can one  identify all existing connections  and estimate all system parameters of a  recurrent network?}

This task becomes extremely challenging in modern big  network applications,  because  the available observations  are usually very noisy and only available at a relatively small number of time points (say $T$),  due to budget or equipment limitations.  One frequently faces  applications with $n^2$ much larger than  $T$.  In addition, in this continuous time setting,  no analytical formula of the likelihood exists for the stochastic model, which increases the estimation difficulty even in large samples  \cite{kou2012multiresolution}.
%Conventional  estimation methods  do not apply well in such situations.
Instead of considering  multi-step {ad-hoc} procedures,  this paper  aims at learning the network system \textit{as a whole}. Multivariate statistical techniques will be developed for identifying  complete  topology and recovering  all dynamical parameters. To the best of our knowledge, automatic topology and dynamics learning in large-scale recurrent networks has not been studied before.
%For some ad-hoc [DOUBLE CHECK THIS; ADD MORE PAPERS in MORE areas???] procedures in, say,  bioinformatics, see \cite{Chen2004, Chen2005}.
%To the best of our knowledge,  automatic topology learning in recurrent networks have not been investigated before.

In this work, we are interested in  networks that are \textbf{sparse} in topology.  First, many real-world complex dynamical networks indeed have sparse or approximately sparse structures. For example, in  regulatory networks, a gene is only regulated by  a handful of others ~\cite{Fujita07}.  Second, when the number of nodes is large or very large  compared with the number of observations, the sparsity assumption reduces the number of model parameters so that the system is estimable. Third, from a philosophical point of view, a sparse network modeling is consistent with the principle of Occam's razor.

Not surprisingly, there is a surge of interest of using compressive sensing  techniques for parsimonious network topology learning and dynamics prediction. However, relying on  sparsity alone seems to have only limited power in  addressing  the difficulties of large-scale network learning from  big data.
To add more prior knowledge and  to  further reduce the number of effective unknowns, we propose to study  how to {incorporate     structural properties of the  network system into  learning and estimation, in addition to sparsity}.
In fact, real-life networks of interest usually demonstrate %(i)  \textbf{sparse} topology, and (ii)
\textbf{asymptotic stability}.
%The first implies that  not many  nodes directly influence a given node significantly, which is a popular assumption in modern network studies and complies with Occam's razor principle.
This is  one of the main reasons why practitioners only perform {limited}  number of measurements of the system, which again provides a  type of parsimony or shrinkage in network learning.

In this paper we  develop sparse sigmoidal network learning algorithms,  with rigorous  convergence guarantee in theory, for a variety of sparsity-promoting penalty forms. A quantile variant,    the progressive recurrent network screening, is proposed for efficient computation and  allows for  direct cardinality control  of network topology in estimation.  Moreover, we investigate recurrent network stability conditions  in \textit{Lyapunov}'s sense, and incorporate such stability constraints into  sparse network learning.
The remaining of this paper is organized as follows.  Section~\ref{sec:model} introduces the sigmoidal recurrent network model,  and formulates a multivariate regularized problem based on the discrete-time approximate likelihood. Section~\ref{sec:algorithm} proposes a class of sparse network learning algorithms based on the study of sparse sigmoidal regressions. A novel and efficient  recurrent network screening (RNS) with theoretical   guarantee  of convergence is advocated  for topology identification in ultra-high dimensions. Section~\ref{sec:stability} investigates  asymptotic stability conditions in recurrent systems,  resulting in  a  stable-sparse sigmoidal ({\sss}) network learning. In Section~\ref{sec:dataexps}, synthetic data experiments and  real applications are given.
All proof details are left to the Appendices.
\section{Model Formulation}
\label{sec:model}
To describe the evolving state of a  continuous-time recurrent neural network,   one usually defines an associated    dynamical system.  Ideally, without any randomness, the network behavior can be specified by a set of ordinary differential equations:
\begin{align}
\frac{\rd x_i}{\rd t} = l_i \pi( \sum_{j\neq i} \alpha_{ij}  x_j + u_i) - d_i x_i + c_i, \quad i=1,\cdots,n,  \label{sigode0}
\end{align}
where $x_i$, short for $x_i(t)$, denotes the dynamic process of node $i$. Throughout the paper,  $\pi$ is the \textit{sigmoidal} activation function
\begin{align}
\pi(\theta) =  \frac{1}{1+e^{-\theta}},
\end{align}
which is   most frequently used in recurrent networks. This  function  is smooth and strictly increasing. It has a proper shape to resemble   many real-world mechanisms and behaviors.
Due to noise contamination%\begin{align*}
%\frac{\rd x_i}{\rd t} = \frac{b_{i1}}{1+\exp(-\alpha_{i0}-\sum_{j\in S_i} \alpha_{ij} x_j)} - b_{i2} x_i,
%\end{align*}
, a stochastic differential equation  model is  more realistic
\begin{align}
\rd x_i = (l_i \pi( \sum_{j\neq i} \alpha_{ij}  x_j + u_i) - d_i x_i + c_i)\rd t + \sigma \rd \mathcal{B}_t,  \label{eq:sigsde_univariate}
\end{align}
where $\mathcal{B}_t$ is a standard Brownian motion and reflects the stochastic nature of the system.
%See, e.g,  \cite{marnellos1998gene}, \cite{vohradsky2001neural}, \cite{vu2007},  for some interesting applications in bio-related fields [ADD more refs??].
Typically   $l_i>0, d_i>0$, and in some applications  $\alpha_{ii}=0$ (no self-regulation) is required.

%Figure~\ref{fig:ex_sig_system} shows an example of a small network consisting of 10 nodes. The topology of the recurrent network is shown on the left. The dynamics of a subnetwork including mode 1, 2 and 3 are shown on the right.
%\begin{figure}[h!]
%  \centering
%      \includegraphics[width=0.8\textwidth]{sigmoid_model_1.eps}
% \caption{Topology and dynamics of the sigmoidal recurrent network model.}
%\label{fig:ex_sig_system}
%\end{figure}
%A linear combination of $x_2$ and $x_3$ goes through a sigmoid activation rule and applies on $x_1$. Such a relationship is called regulation, either excitatory (if the coefficient is positive) or inhibitory (if the coefficient is negative), and node 2 and node 3 are called regulators of node 1. The regulatory relations are reflected by the topology of the recurrent network, which tends to be very sparse in practice. For example, a stock index is significantly regulated by only a small number of other indices in the stock market. It is a challenging task in statistical inference to determine the regulator set for each node.
%Current methods compute the marginal relations between each pair of nodes without taking into account their interactions with the rest of the network. Rather, we propose to take a \textbf{multivariate} approach and infer the regulatory relations within the network altogether.

In the sigmoidal recurrent network model, the coefficients $\alpha_{ij}$ characterize    the between-node interactions.  In particular, $\alpha_{ij}=0$ indicates that node $j$ does not \textit{directly} influence  node $i$; otherwise, node $j$ regulates node $i$, and is referred to as  a regulator of node $i$ in gene regulatory networks. Such a regulation relationship can be either excitatory (if $\alpha_{ij}$ is positive) or inhibitory (if $\alpha_{ij}$ is negative). In this way, $\bsbA=[\alpha_{ij}]$ is associated  with a directed graph that captures the Granger causal relationships between all nodes \cite{Granger69}, and the topology of the recurrent network is revealed by the zero-nonzero pattern of  the matrix.
Therefore, to identify  all significant regulation links, it is of great interest  to estimate $\bsbA$,   given a sequence of  node observations (snapshots of the network).

% We introduce some vector/matrix symbols to simplify the expression.
Denote the system state at time $t$ by $\bsbx(t)=[ x_1(t) \cdots  x_n(t)]^\tran$ or   $\bsbx_t$ (or simply  $\bsbx$  when there is no ambiguity.)  Define $\bsbl = [ l_1,    \cdots l_n ]^\tran$,
$\bsbu = [ u_1, \cdots, u_n ]^\tran$, $\bsbc = [  c_1, \cdots, c_n ]^\tran$, $\bsbD = \mbox{diag}\{d_i\}$,  $\bsbL = \mbox{diag}\{l_i\}$, and $\bsbA = [\alpha_{ij}] =[\bsba_1 \cdots  \bsba_n]^{\tran} \in \mathbb R^{n\times n}$. Then \eqref{eq:sigsde_univariate} can be  represented in a \textit{multivariate}  form
\begin{align}
 \rd \bsbx_t = \left(\bsbL \bsbpi( \bsbA  \bsbx + \bsbu) - \bsbD \bsbx + \bsbc\right) \rd t + \sigma \rd \bsb{\mathcal{B}}_t,    \label{eq:sigsde}
\end{align}
where $\bsb{\mathcal B_t}$ is an $n$-dimensional standard Brownian motion.
%It effectively captures the nondeterministic nature of a dynamic system and the random noises in the observations.
While this model is specified in continuous time,
in practice, the observations are always collected at discrete time points. Estimating the parameters of  an SDE model from  few discrete observations is very challenging. There rarely exists an analytical expression of the exact likelihood. A common treatment  is to discretize  ~\eqref{eq:sigsde} and use an approximate likelihood instead. We use the Euler discretization scheme (see, e.g., \cite{bally1996law}):
\begin{equation}
\Delta \bsb{x} = (\bsb{L}\bsb{\pi}(\bsb{A}\bsb{x}+\bsb{u})
-\bsb{D}\bsb{x}+\bsb{c})\Delta t + \sigma \Delta \bsb{\mathcal{B}}_t.
\label{eq:sigmoid_discret} \nonumber
\end{equation}

Suppose the system \eqref{eq:sigsde} is observed at $T+1$ time points $t_1, \cdots, t_{T+1}$.
Let
%\begin{align}
%\tilde \bsbx_i=[\begin{array}{ccc} x_i(t_1) & \cdots & x_i(t_{T}) \end{array}]^{\tran} \in \mathbb R^\tran
%\end{align}
% be the observed values of node $i$ at  time points from $t_1$ to $t_T$, and
%\begin{align}
$\bsbx_s =[ x_1(t_s),  \cdots,  x_n(t_s) ]^{\tran}  \in \mathbb R^n
$
%\end{align}
 be the observed values of all $n$ nodes at $t_s$.
%This data matrix
%can be represented by $\left[\begin{array}{ccc} \tilde \bsbx_1  & \cdots & \tilde \bsbx_n \end{array}\right]$ or $\left[\begin{array}{c}  \bsbx_1^\tran \\\vdots \\  \bsbx_{T}^\tran\end{array}\right].$
Define $\Delta  \bsbx_s = (\bsbx_{s+1} -  \bsbx_s)$ and $\Delta t_s = (t_{s+1} - t_s)$, $1\leq s \leq T$.  Because $\Delta\bsb{\mathcal{B}_t}\sim\mathcal{N}(\bsb{0},\Delta t\bsbI)$, the negative {conditional} log-likelihood for the {discretized} model is given by
\begin{align*}
&\ell(\bsbx_1, \cdots, \bsbx_{T+1} |  \bsbx_1) = - \log   P( \Delta\bsbx_1, \cdots, \Delta  \bsbx_{T} |   \bsbx_1)\\
=& - \log   P(\Delta   \bsbx_1 |   \bsbx_1) \cdots P(\Delta   \bsbx_{T} |   \bsbx_{T} )\\
=&  \sum_{s=1}^{T} \|{\Delta   \bsbx_s}/{\Delta t_s} - (\bsbL \bsbpi( \bsbA    \bsbx_s + \bsbu) - \bsbD   \bsbx_s + \bsbc)\|_2^2 \Delta t_s/(2\sigma^2) + C(\sigma^2) \\
=: &f(\bsbA, \bsbu, \bsbl, \bsb{d}, \bsbc)/\sigma^2 + C(\sigma^2).
\end{align*}
$C(\sigma^2)$ is a function that depends on $\sigma^2$ only.
The fitting criterion   $f$ is separable in $\bsba_1, \ldots, \bsba_n$. To see this, let $x_{i, s} = x_i(t_{s})$ and $\Delta x_{i, s} = x_i(t_{s+1}) - x_i(t_{s})$. Then, it is easy to verify that
\begin{align}
 \begin{split}
 f(\bsbA, \bsbu, \bsbl, \bsb{d}, \bsbc)  =  \frac{1}{2 } \sum_{i=1}^n \sum_{s=1}^{T} \left(\frac{\Delta x_{i,s}}{\Delta t_s} - (l_i  \pi(  \bsba_i^{\tran} \bsbx_s  + u_i) - d_i  x_{i,s} + c_i)\right)^2 \Delta t_s. \label{loss-sumform}
\end{split}\end{align}
Conventionally, the unknown parameters can then be estimated by minimizing $f$. In modern applications, however, the number of available observations ($T$) is often much smaller than the number of variables to be estimated ($n^2+4n$), due to, for example, equipment/budget limitations. Classical  MLE methods  do not apply well in this high-dimensional setting.

Fortunately,  the networks of interest in reality often possess  topology sparsity. For example, a stock price may not be directly influenced by all the other stocks in the stock market. A parsimonious network with only  significant regulation links is much more interpretable. Statistically speaking, the sparsity in $\bsbA$ suggests the necessity of   \textbf{shrinkage estimation} \cite{JamesSteinEstimator} which can be  done by  adding penalties and/or constraints to the loss function. The general penalized maximum likelihood problem is
\begin{align}
\label{eq:pml}
 \min_{\bsbA, \bsbu, \bsbl, \bsb{d}, \bsbc} & f(\bsbA, \bsbu, \bsbl, \bsb{d}, \bsbc) + \sum_{i, j} P(\alpha_{ij}; \lambda_{ji})
\end{align}
where $P$ is a penalty  promoting sparsity and $\lambda_{j i}$ are regularization parameters.
Among the very many possible choices of $P$,  the $\ell_1$ penalty is perhaps the most popular to enforce sparsity:
\begin{align}
P(t; \lambda) = \lambda |t|.  \label{eq:l1-pen}
\end{align}
It provides a convex relaxation of the  $\ell_0$ penalty
\begin{align}
P(t; \lambda) = \frac{\lambda^2}{2} 1_{t\neq 0}.   \label{eq:l0-pen}
\end{align}
Taking both  topology identification and dynamics prediction into consideration,  we are particularly interested in the $\ell_0+\ell_2$ penalty \cite{SheGLMTISP}
\begin{align}
P(t; \lambda, \eta)   = \frac{1}{2}\frac{\lambda^2}{1+\eta} 1_{t\neq 0} + \frac{\eta}{2}t^2,  \label{eq:l02-pen}
\end{align}
where the $\ell_2$ penalty or Tikhonov regularization can effectively deal with large noise and collinearity \cite{zou2005regularization,Park07} to  enhance estimation accuracy.

The shrinkage estimation problem  \eqref{eq:pml} is  however nontrivial. The loss $f$ is  nonconvex, $\pi$ is nonlinear, and the penalty $P$ may be nonconvex or even discrete, let alone the high-dimensionality challenge. Indeed,  in many practical networks, the available observations are usually quite limited and noisy. Effective and efficient learning algorithms are in great need to meet the modern  big data challenge.

\section{Sparse Sigmoidal Regression for Recurrent Network Learning}
\label{sec:algorithm}
\subsection{Univariate-response sigmoidal regression }
\label{sec:univariate}
As analyzed  previously, to solve \eqref{eq:pml}, it is sufficient to study a univariate-response learning problem
\begin{align}
\begin{split}
\min_{(\bsbb, l, \bsbg)} &\frac{1}{2}\sum_{s=1}^T  w_s \left(y_s - l  \pi(\tilde\bsbx_s^{\tran}\bsbb)  - \bsbz_s^{\tran} \bsbg \right)^2 + \sum_{k=1}^n P(\beta_k, \lambda_k)=: F(\bsbb, l, \bsbg)   \label{eq:sigmoidopt}
\end{split}
\end{align}
where  $\tilde\bsbx_s, \bsbz_s, y_s, w_s$ are given, and $l, \bsbg, \bsbb$ are unknown with $\bsbb$  desired to be sparse.
\eqref{eq:sigmoidopt} is  the recurrent network learning problem for node $i$  when
we set $\tilde{\bsbx}_s=[1, \bsbx_s]^\tran$,
$y_s= \Delta  x_{i,s} /\Delta t_s$, $\bsbz_s=[1,x_{i,s}]^\tran$,  and $w_s = \Delta t_s$ ($1\leq s \leq T$) (note that the intercept $\beta_1$ is usually subject to no penalization corresponding to  $\lambda_1=0$). For notational ease, define $\tilde\bsbX=[\tilde\bsbx_1,\cdots, \tilde\bsbx_T]^\tran$, $\bsby=[ y_1, \cdots,  y_T ]^\tran$,  $\bsbZ=[\bsbz_1,  \cdots,  \bsbz_T]^\tran$,  $\bsb{\lambda}= [\lambda_k] \in \mathbb R^{n}$, and   $\bsbW =\mbox{diag}\{\bsbw\}=\mbox{diag}\{w_1, \cdots, w_s\}$ (to be used in this subsection only, unless otherwise specified).

%When the design matrix, the response vector and the $\bsbZ$-matrix in~\eqref{eq:sigmoidopt}
%are given by $\bsbX=[\begin{array}{cccc} \bsb{1} & \bsbx_1  & \cdots & \bsbx_{T} \end{array} ]^\tran$, $\bsby = [ \Delta  x_{i,s} /\Delta t_s ]_{s=1}^{T}$,    $\bsbZ = \left[\begin{array}{cc} \bsb{1} & \tilde \bsbx_i \end{array} \right]$ respectively,  and $w_s = \Delta t_s$ ($1\leq s \leq T$), \eqref{eq:sigmoidopt} becomes  the learning problem for node $i$ mentioned previously.
%In this section, we represent  $\bsbX=[\begin{array}{ccc}  \bsbx_1  & \cdots & \bsbx_T \end{array} ]^\tran$, $\bsby=[\begin{array}{ccc} y_1  & \cdots & y_T \end{array} ]^\tran$, and $\bsbZ=[\begin{array}{ccc} \bsbz_1  & \cdots & \bsbz_T \end{array}]^\tran$ by a little abuse of notation.

We propose a simple-to-implement and efficient  algorithm to solve the general optimization problem in the form of \eqref{eq:sigmoidopt}.
First define two useful auxiliary functions
\begin{eqnarray}
\xi(\theta, y) &=& \pi(\theta) (1-\pi(\theta)) (\pi(\theta) - y), \label{xidef}\\
k_0(\bsby, \bsbw) &=& \max_{1\leq s \leq T}  \frac{{w_s}}{16} \left(1+ \frac{ (1-2y_s)^2}{2}\right).\label{k0def}
\end{eqnarray}
The vector versions of $\pi$ and $\xi$ are defined componentwise: $\bsbpi(\bsb{\theta})=[\pi(\theta_1), \cdots, \pi(\theta_T)]^{\tran}$ and
$\bsbxi(\bsb{\theta}, \bsb{y})=[\xi(\theta_1,y_1), \cdots, \xi(\theta_T, y_T)]^{\tran}$, and the matrix versions are defined similarly.
A prototype algorithm is described as follows, starting with an initial estimate $\bsbb^{(0)}$,   a thresholding rule $\Theta$ (an odd,
shrinking and nondecreasing function, cf. \cite{SheIPOD}), and $j=0$.

\begin{algorithmic}
%\STATE  $j\gets 1$, $\mathrm{converged}\gets\mathrm{FALSE}$
%\WHILE{not converged}
\REPEAT
\STATE \textit{0) } $j\gets j+1$
\STATE \textit{1) } $\bsbmu^{(j)} \gets \bsbpi(\tilde\bsbX\bsbb^{(j-1)})$
\STATE \textit{2) } Fit a  weighted  least-squares   model
%$\bsby\sim l \bsbmu^{(j)} + \bsbZ \bsbg$ with the observation weights given by $\{w_s\}$.
\begin{equation}
\min_{l,\bsbg} \| \bsbW^{1/2} (\bsby- [\bsbmu^{(j)} \ \bsbZ ] [ l \  \bsbg^\tran]^\tran )\|_2^2, %(\bsby-l \bsbmu^{(j)}-\bsbZ \bsbg)\label{eq:quadratic}
\end{equation}
with the corresponding solution  denoted by $(l^{(j)}, \bsbg^{(j)})$.
\STATE \textit{3) } Construct $\tilde\bsby^{(j)} \gets (\bsby- \bsbZ \bsbg^{(j)})/ l^{(j)}$,  $\tilde\bsbw^{(j)} \gets  (l^{(j)})^2 \bsbw$,  $\tilde\bsbW^{(j)} = \mbox{diag}\{\tilde\bsbw^{(j)}\}$. Let   $K^{(j)}$ be any constant no less than $k_0(\tilde \bsby^{(j)}, \tilde \bsbw^{(j)})  \|\tilde \bsbX\|_2^2$, where $\|\cdot\|_2$ is the spectral norm.
\STATE \textit{4) } Update  $\bsbb$ via  thresholding:
\begin{equation}
\label{eq:uni_th}
\bsbb^{(j)} =   \Theta ( \bsbb^{(j-1)} - \frac{1}{K^{(j)}} \tilde\bsbX^{\tran} \tilde\bsbW^{(j)} \bsbxi(\tilde\bsbX\bsbb^{(j-1)}, \tilde \bsby^{(j)}); \bsb{\lambda}^{(j)} ),
\end{equation}
where $\bsb{\lambda}^{(j)}= [\lambda_k^{(j)}]_{n\times 1}$ is a scaled version of    $\bsb{\lambda}$ satisfying $P(t; \lambda_{k})/ K^{(j)} = P(t; \lambda_k^{(j)})$ for any $t\in \mathbb R$ , $1\leq k \leq n$.
%\STATE \textit{5) } $\mathrm{converged}\gets \Vert  \bsbb^{(j+1)}-  \bsbb^{(j)}\Vert_\infty<\varepsilon$
%\STATE \textit{5) } $j\gets j+1$
%\ENDWHILE %
%UNTIL{$\|\bsbb^{(j)}-\bsbb^{(j-1)}\|<\varepsilon$}
%\STATE\deliver\ $\hat \bsbb = \bsbb^{(j)}$, $\hat l = l^{(j)}$, and $\hat \bsbg = \bsbg^{(j)}$.
\UNTIL{convergence}
\end{algorithmic}

 Before proceeding, we give some examples of $\Theta$ and $\bsb{\lambda}^{(j)}$. The specific form of the thresholding function $\Theta$ in \eqref{eq:uni_th} is related to  the penalty $P$  through the following formula \cite{SheGLMTISP}:
\begin{equation}
\begin{split}
&P(t;\lambda)-P(0;\lambda) \\
=& \int_0^{|t|}( \mbox{sup}\{s: \Theta(s;\lambda)\leq u\}-u )\rd u
+ q(t;\lambda)
\end{split} \label{eq:construction}
\end{equation}
with $q(\cdot;\lambda)$ nonnegative and $q(\Theta(s;\lambda))=0$ for all $s$.
 The regularization parameter(s) are rescaled at each iteration according to the form of $P$.
Examples include:   (i) the  $\ell_1$ penalty \eqref{eq:l1-pen}, and its associated \textit{soft-thresholding}
$
\Theta_S(t; \lambda) = \sgn(t)(|t|-\lambda)1_{|t|> \lambda},
$
in which case  $P(t; \lambda)/ K^{(j)} = P(t; \lambda^{(j)}), \forall t$ implies $\lambda^{(j)}=\lambda/K^{(j)}$,   (ii) the $\ell_0$ penalty \eqref{eq:l0-pen} and  the \textit{hard-thresholding}
$
\Theta_H(t)= t 1_{|t|> \lambda},
$ which determines  $\lambda^{(j)}=\lambda/\sqrt{K^{(j)}}$,
and  (iii) the $\ell_0+\ell_2$ penalty \eqref{eq:l02-pen} and its associated
 \textit{hard-ridge thresholding}
$
\Theta_{HR}(t; \lambda,\eta)=\frac{t}{1+\eta}1_{|t|> \lambda},
$
where $\lambda^{(j)}=\frac{\lambda}{K^{(j)}}\sqrt{{(\eta+K^{(j))}}/{(\eta+1)}}, \eta^{(j)}=\eta/K^{(j)}$.
The $\ell_p$ ($0<p<1$) penalties,   the elastic net penalty, and others \cite{SheGLMTISP}   are also instances of this framework.
\begin{theorem}
\label{th:fValDecrease}
Given the objective function in \eqref{eq:sigmoidopt}, suppose $\Theta$ and  $P$ satisfy \eqref{eq:construction}  for some  nonnegative $q(t;\lambda)$ with $q(\Theta(t;\lambda))=0$ for all $t$ and $\lambda$.
Then given any initial point $\bsbb^{(0)}$,  with probability 1 the sequence of iterates $(\bsbb^{(j)}, l^{(j)}, \bsbg^{(j)})$ generated by the prototype algorithm satisfies
\begin{align}
F(\bsbb^{(j-1)}, l^{(j-1)}, \bsbg^{(j-1)}) \geq F(\bsbb^{(j)},l^{(j)},  \bsbg^{(j)}). \label{funcvaldes}
\end{align}
\end{theorem}
See Appendix~\ref{app:fValDecrease} for the  proof.

Normalization is usually necessary to make all predictors equally span in the space  before  penalizing their coefficients using a single regularization parameter $\lambda$. We can center and scale all predictor columns but the intercept before  calling the algorithm. Alternatively, it is sometimes more convenient to specify $\lambda_k$ componentwise---e.g., in  the $\ell_1$ penalty $\sum \lambda_k |\beta_k|$ set $\lambda_k=\lambda \cdot \|\bsbX[:, k]\|_2$ for non-intercept coefficients.

\subsection{Cardinality constrained  sigmoidal network learning }
In the recurrent network setting,  one can directly apply the prototype algorithm  in Section~\ref{sec:univariate}  to solve~\eqref{eq:pml}, by updating the columns in $\bsbB$ one at a time. On the other hand, a {multivariate} update form is usually more efficient and convenient in implementation. Moreover, it    facilitates   integrating stability  into network learning (cf. Section~\ref{sec:stability}).

To formulate the loss in a multivariate form, we introduce
\begin{align*}
\bsbY &= [y_{i,s}]=[\Delta x_{i,s}/\Delta t_s]\in \mathbb R^{T\times n},
\\\bsbX &=[\bsbx_1,\cdots, \bsbx_T]^\tran\in \mathbb R^{T\times n},
\\
\bsbB &=\bsbA^\tran= [\bsba_1, \cdots, \bsba_n]\in \mathbb R^{n\times n}, \\ \bsbW &=\mbox{diag}\{\bsbw\}= \mbox{diag}\{ \Delta t_1, \cdots, \Delta t_T\}.
\end{align*} Then  $f$ in \eqref{loss-sumform} can be rewritten as
\begin{align}
\label{eq:negloglik}
 f(\bsbB, \bsbu, \bsbl, \bsb{d}, \bsbc)  =  \frac{1}{2} \| \bsbW^{1/2} \left\{ \bsbY - [\bsbpi(\bsbX \bsbB + \bsb{1} \bsbu^\tran) \bsbL - \bsbX \bsbD + \bsb{1} \bsbc^{\tran} ]\right\} \|_F^2,
\end{align}
with $\| \cdot \|_F$ denoting the Frobenius norm,
and the objective  function to  minimize becomes
$%\begin{align}\label{eq:multi_obj}
 \| \bsbW^{1/2} \{ \bsbY - [\bsb{\pi}(\bsbX \bsbB + \bsb{1} \bsbu^\tran) \bsbL - \bsbX \bsbD + \bsb{1} \bsbc^{\tran} ]\} \|_F^2 + P(\bsbB,\bsbLamb).
%\end{align}
$

One of the main issues is to choose a proper penalty form for sparse network learning.
Popular sparsity-promoting penalties include  $\ell_1$, SCAD, $\ell_0$, among others. The  $\ell_0$ penalty \eqref{eq:l0-pen} is  ideal in  pursuing a parsimonious solution. However,   the matter of parameter tuning  cannot be ignored.   Most tuning strategies (such as  $K$-fold cross-validation)  require computing a solution path  for a grid of values of $\lambda$, which is  quite time-consuming  in large network estimation. Rather than applying   the  $\ell_0$ penalty, we propose  an  $\ell_0$ \textbf{constrained}
 sparse network learning \begin{equation}
\label{eq:l0_constraint}
\|\bsbB\|_0 \leq m,
\end{equation}
where $\|\cdot\|_0$ denotes the number of nonzero components in a matrix. In contrast to the penalty parameter $\lambda$ in  \eqref{eq:l0-pen}, $m$ is more meaningful and customizable. One can  directly specify its value based on prior knowledge or availability of computational resources to have  control of the network connection cardinality. To account for  collinearity and large noise contamination, we   add a further  $\ell_2$ penalty in $\bsbB$, resulting in a  new `$\ell_0+\ell_2$' regularized criterion
\begin{equation}
\begin{gathered}
\min_{\bsbB,\bsbL,\bsbD, \bsbu, \bsbc} \frac{1}{2} \| \bsbW^{1/2} \{ \bsbY - [\bsb{\pi}(\bsbX  \bsbB + \bsb{1} \bsbu^{\tran}) \bsbL - \bsbX \bsbD + \bsb{1} \bsbc^{\tran} ]\} \|_F^2\\ + \frac{\eta}{2}\|\bsbB\|_F^2=:F_0,  \mbox{ s.t. } \|\bsbB\|_0  \leq m.
\end{gathered}
\label{eq:constraintForm}
\end{equation}
Not only is the cardinality  bound $m$  convenient  to set, because of the sparsity assumption, but    the $\ell_2$ shrinkage parameter  $\eta$ is easy to tune. Indeed,  $\eta$ is usually not a very sensitive parameter and   does not require a large grid of candidate values. Many researchers simply fix  it at a small value (say, $1e-3$) which  can  effectively reduce the prediction error (e.g., \cite{Park07}).
Similarly, to handle the possible collinearity between  $\bsbpi$ and  $\bsb{1}$, we recommend adding mild ridge penalties  $\frac{\eta_l}{2}  \| \bsbl\|_2^2 + \frac{\eta_c}{2}\|\bsbc\|_2^2$ in \eqref{eq:constraintForm} (say, $\eta_l = 1e-4$ and $\eta_c=1e-2$).

The {constrained} optimization of  \eqref{eq:constraintForm} does not apparently fall into  the penalized framework proposed in Section \ref{sec:univariate}. However, we can adapt the technique to handle  it through a \textit{quantile} thresholding operator (as a variant of the hard-ridge thresholding \eqref{eq:l02-pen}).
The detailed recurrent network screening (\textbf{RNS})  algorithm  is   described as follows, where for notational simplicity    $\tilde \bsbX:= [\bsb{1} \  \bsbX]$, $\tilde\bsbB:=[\bsbu \  \bsbB^\tran]^\tran$, and we denote  by $\bsbA[I, J]$ the submatrix of $\bsbA$ consisting of the rows and columns indexed by $I$ and $J$, respectively.
\begin{algorithm}[h!]
{
\begin{algorithmic}
\caption{\ Recurrent Network Screening (\textbf{RNS})} \label{alg:sigscr}
\STATE \given\ $\tilde\bsbB^{(0)}$ (initial estimate), $m$ (cardinality bound), $\eta$ ($\ell_2$ shrinkage parameter). %$\eta$ (ridge parameter);
\STATE  $j\gets 0$
\REPEAT
\STATE \textit{\textbf{0)} } $j\gets j+1$
\STATE \textit{\textbf{1)} } $\bsbmu^{(j)} \gets \bsbpi(\tilde\bsbX\tilde\bsbB^{(j-1)})$
\STATE \textit{\textbf{2)} } Update $\bsbL =\mbox{diag}\{\bsbl\},\bsbD=\mbox{diag}\{\bsb{d}\}$, and $\bsbc$ by fitting a    weighted vector  least-squares   model:
%$\bsbY\sim  \bsbmu^{(j)} \bsbL - \bsbX \bsbD + \bsb{1} \bsbc^\tran$ with the observation weights given by $\bsbW$. Explicit solutions can be obtained, say, through fitting   $n$ weighted least squares.
\begin{equation}
\min_{\bsbL,\bsbD,\bsbc} \|\bsbW^{1/2}(\bsbY-[\bsbmu^{(j)} \bsbL - \bsbX \bsbD + \bsb{1} \bsbc^\tran])\|_F^2,
\label{eq:quadratic_multi}
\end{equation}
 with the solution denoted  by $\bsbL^{(j)}=\mbox{diag}\{\bsbl^{(j)}\}=\mbox{diag}\{l_1^{(j)},\cdots, l_n^{(j)}\}$,  $\bsbD^{(j)}=\mbox{diag}\{d_1^{(j)},\cdots, d_n^{(j)}\}$, and $\bsbc^{(j)}$. (This amounts to solving $n$ separate weighted least squares problems.)

\STATE \textit{\textbf{3)} } Construct $\tilde\bsbY \gets (\bsbY+  \bsbX \bsbD^{(j)} - \bsb{1} (\bsbc^{(j)})^\tran) (\bsbL^{(j)})^{-}$,  $\tilde\bsbW \gets   \bsbw \cdot(\bsbl^{(j)} \circ \bsbl^{(j)})^{ \tran}$, where $^-$ denotes the Moore-Penrose  pseudoinverse and $\circ$ is the Hadamard product. Let  $K_i^{(j)}$ be any constant no less than  $k_0(\tilde \bsbY[:,i], \tilde\bsbW[:,i]])  \|\tilde \bsbX\|_2^2$, and $\bsbK^{(j)}\gets\mbox{diag}\{K_1^{(j)}, \cdots, K_n^{(j)}\}$.
\STATE \textit{\textbf{4)} } Update $\bsbB$ and $\bsbu$: \\
\textbf{\textit{4.1)}} $\tilde{\bsb{\Xi}} \gets  \tilde\bsbB^{(j-1)} -  {\tilde{\bsbX}}^{\tran} \{\tilde\bsbW \circ \bsbxi(\tilde\bsbX\tilde\bsbB^{(j-1)}, \tilde\bsbY)\}(\bsbK^{(j)})^{-}$, $\bsbu^{(j)}\gets (\tilde{\bsb{\Xi}}[1, :])^\tran$, ${\bsb{\Xi}}\gets \tilde{\bsb{\Xi}}[2:\mbox{end, :}]$ (the submatrix of $\tilde \bsbXi$ without the first row), $\bsb{\eta}^{(j)}=\eta \cdot (\bsb{1}\bsb{1}^{\tran})(\bsbK^{(j)})^-$.\\
\textbf{\textit{4.2)}} Perform the hard-ridge thresholding  $\bsbB^{(j)} \gets \Theta_{HR} \left( {\bsb{\Xi}};  \bsb{\zeta}({\bsb{\Xi}}),  \bsb{\eta}^{(j)} \right)$ (cf. Section \ref{sec:univariate} for  the definition of $\Theta_{HR}$)
with an adaptive threshold
 matrix $\bsb{\zeta}({\bsb{\Xi}})$. The entries of $\bsb{\zeta}(\bsbXi)$ are all set to the medium of  the $m$th  and the $(m+1)$th largest components  of $\mbox{vec}(|{\bsb{\Xi}}|)$. See \eqref{thconstr}  for  other variants  when certain links must be maintained or forbidden. \\
\textbf{\textit{4.3)}} $\tilde\bsbB^{(j)}\gets [\bsbu^{(j)} \ \ (\bsbB^{(j)})^\tran]^\tran$

% ${\bsb{\Lambda}}^{(j)} =  \left[\lambda^{(j)}_{u,v}\right]_{n\times n}$ with
%  $\lambda^{(j)}_{u,v}$ for $u\neq v$ being the medium of  the $m$th largest component and the $(m+1)$th largest component of $\{|\tilde{b}^{(j+0.5)}_{u+1,v}|\}_{1\leq u\leq n, 1\leq v\leq n, u\neq v}$ and $\lambda^{(j)}_{u,v}$ for $u=v$ being $\infty$ (no self-regulation).
%\STATE \textit{5) } $\mathrm{converged}\gets \Vert  \tilde\bsbB^{(j+1)}-  \tilde\bsbB^{(j)}\Vert_{\max}<\varepsilon$
%\STATE \textit{5) } $j\gets j+1$
\UNTIL{the decrease in function value  is small} %{ $\|\bsbB^{(j)}-\bsbB^{(j-1)}\|$ is small }
\STATE\deliver\ $\hat\bsbB = \bsbB^{(j)}$, $\hat\bsbu=\bsbu^{(j)}$, $\hat\bsbL = \bsbL^{(j)}$,  $\hat \bsbD = \bsbD^{(j)}$,  $\hat\bsbc=\bsbc^{(j)}$.
\end{algorithmic}
}
\end{algorithm}

\begin{theorem}
\label{th:constraintform}
Given any initial point $\tilde\bsbB^{(0)}$, Algorithm~\ref{alg:sigscr}
converges in the sense that with probability 1, the function value decreasing property holds:
\begin{align*}
F_{0}(\bsbB^{(j-1)},\bsbL^{(j-1)},\bsbD^{(j-1)}, \bsbu^{(j-1)}, \bsbc^{(j-1)})
 \\\geq   F_{0}(\bsbB^{(j)},\bsbL^{(j)},\bsbD^{(j)}, \bsbu^{(j)}, \bsbc^{(j)}),
\end{align*}
and all $\bsbB^{(j)}$ satisfy $\| \bsbB^{(j)}\|_0\leq m$, $\forall j\geq 1$. %When $\eta>0$,  the sequence of iterates strictly converges.
\end{theorem}
See Appendix~\ref{app:constraintform} for the  proof.

Step 2 can be implemented by solving $n$ weighted least squares.   Or, one can re-formulate it as a single-response problem to obtain the solution  $(\bsbl^{(j)}, \bsb{d}^{(j)}, \bsbc^{(j)})$ in one step. The latter way is usually more efficient. When there are ridge penalties  imposed on $\bsbl$ and $\bsbc$, the computation is similar.

In certain applications   self-regulations are  not allowed. Then  $\zeta({\bsb{\Xi}})$ should be   the medium of  the $m$th  and the $(m+1)$th largest  elements  of $|\bsbXi - \bsbXi \circ \bsb{I}|$  ($\bsb{\Xi}$ in absolute value after excluding its diagonal entries).\footnote{Throughout the paper   $| \bsbA|$ is the  absolute value of the elements of $\bsbA$. That is, for $\bsbA = [a_{ij}]$,  $|\bsbA|=[|a_{ij}|]$.} Other prohibited links can be similarly treated in determining the dynamic threshold.
In general, given $\mathcal T$  the index set of the links that must be maintained and $\mathcal F$  the index set of the links that are forbidden, the threshold is constructed as follows
\begin{align}
\begin{split}
&\bsb{\Xi}[i] \leftarrow 0, \forall i \in \mathcal F \\
&\bsb{\zeta}[i] \leftarrow \begin{cases} 0, & \mbox{ if } i \in \mathcal T \\  \left| \bsbXi[\mathcal T^c] \right|_{(m+1)}, &\mbox{ otherwise, }\end{cases}
\end{split}\label{thconstr}
\end{align}
where $\left|\bsbXi[\mathcal T^c]\right|_{(m+1)}$ is the $(m+1)$th largest element (in absolution value) in $\bsbXi[\mathcal T^c]$ (all entries of  $\bsb{\Xi}$ except those indexed by $\mathcal T$).
It is easy to show that  the   convergence result still   holds based on the argument in Appendix~\ref{app:constraintform}.

In implementation, we advocate reducing the network cardinality  in a \textbf{progressive} manner  to lessen greediness:     at the $j$th step, $m$ is replaced by $m(j)$, where $\{m(j)\}$  is a  monotone sequence of integers decreasing from $n(n-1)$ (assuming no self-regulation)   to $m$.   Empirical study shows that the sigmoidal decay cooling schedule $m{(j)}= \big\lceil 2n(n-1)/(1+e^{\alpha j}) \big\rceil$ with  $\alpha=0.01$, works well.

RNS  involves no  costly  operations at each iteration, and is simple to implement. It  runs efficiently  for large networks.
The  RNS estimate can be directly used for  analyzing the  network topology. One can also use it for  screening   (in which case  a relatively large value of $m$ is specified), followed by  a fine network learning algorithm restricted on the screened connections. In either case RNS substantially reduces the search space of candidate links.

\section{Stable-Sparse Sigmoidal  Network Learning}
\label{sec:stability}
For  a general dynamical system possibly nonlinear, stability is one of the most fundamental issues \cite{sastry1999nonlinear,lasalle1976stability}.  If a system's  equilibrium point is  \emph{asymptotically stable},  then the perturbed system  must approach the equilibrium point as   $t$ increases. Moreover, one of the main reasons many real network applications only have  limited number of observations measured after  perturbation  is that the associated dynamical systems stabilize fast (e.g., exponentially fast). This offers another important type of parsimony/shrinkage in network  parameter estimation.

To design a new type of regularization, we first investigate  stability conditions of sigmoidal recurrent networks  in Lyapunov's sense \cite{lyapunov1892general,lyapunov1992general}. Then, we develop a stable  sparse sigmoidal (\sss) network learning approach.

\subsection{Conditions for asymptotic stability }
\label{subsec:stabconds}
Recall the multivariate representation of  \eqref{sigode0}
\begin{align}
\frac{\rd \bsbx}{\rd t} = \bsbL \bsbpi( \bsbA  \bsbx + \bsbu) - \bsbD \bsbx + \bsbc.    \label{eq:sigode1_recall}
\end{align}
Because  $\bsbA$ is sparse and typically  singular and  the degradation rates $d_i$ are not necessarily  identical, in general    \eqref{eq:sigode1_recall}  can not be treated as an instance of  the Cohen-Grossberg neural networks \cite{Cohen-Gross1983}. We must first derive its own stability conditions to be considered in  network  estimation.

% \cite{grossberg1988nonlinear,wang2002exponential}.
% Even if $\bsbA$ is nonsingular, the $S$-$\Sigma$ exchange \cite{grossberg1988nonlinear,wang2002exponential} yields for $\bsby = \bsbA  \bsbx + \bsbu$
% \begin{align}
% \frac{\rd \bsby}{\rd t}& = \bsbA \bsbL\bsbpi( \bsby) - \bsbA \bsbD \bsbx + \bsbA\bsbc \notag\\
% & = \bsbA \bsbL\bsbpi( \bsby) - \bsbA \bsbD \bsbA^{-1} \bsby + (\bsbA\bsbc + \bsbA \bsbD \bsbA^{-1} \bsbu).   \label{eq:sigode1v}
% \end{align}
%One cannot directly borrow the conclusions for Cohen-Grossberg networks to study the stability.

Hereinafter,  $\bsbA \succeq \bsbA'$ and $\bsbA \succ  \bsbA'$  stand for the positive  semi-definiteness and positive definiteness of   $\bsbA - \bsbA'$,  respectively, and the set of eigenvalues of $\bsbA$ is $\mbox{spec}(\bsbA)$.
Our conditions are stated below:
\begin{align}
 \bsbD  & \succ \bsb{0}, \tag{A1} \label{nonnegD}\\
 \bsbL &  \succeq \bsb{0}, \tag{A2} \label{nonnegL}\\
% \bsbL & \succ \bsb{0}, \tag{A2b} \label{posL}\\
 \mbox{Re}(\lambda) & < 0 \mbox{ for any } \lambda \in \mbox{Spec}( \bsbL \bsbA  - 4\bsbD) \tag{A3a}, \label{psdcond0}\\
 \bsbL   \bsbA/2 + \bsbA^{\tran}  \bsbL/2 &  \prec  4 \bsbD, \tag{A3b} \label{psdcond1}
% \\ \bsbA/2 + \bsbA^{\tran}/2  & \prec  4 \bsbL^{-1} \bsbD. \tag{A3c} \label{psdcond2}
\end{align}

\begin{theorem}
\label{th:stability}
Suppose  \eqref{nonnegD} \& \eqref{nonnegL} hold. Then \eqref{psdcond0} guarantees the network defined by \eqref{eq:sigode1_recall} has a unique equilibrium point $\bsbx^*$ and is globally  {exponentially} stable in the sense that $\|\bsbx(t) - \bsbx^* \|_2^2 \leq e^{-\varepsilon t}  \|\bsbx(0) - \bsbx^* \|_2^2$ for any solution $\bsbx(t)$. The same conclusion holds  if \eqref{psdcond0} is replaced by  \eqref{psdcond1}.
\end{theorem}
% \begin{corollary}
% \label{cor:stability}
% The same conclusion holds for the dynamic system \eqref{eq:sigode1_recall}, if \eqref{psdcond0} is replaced by \eqref{psdcond1}, or if \eqref{nonnegL} and \eqref{psdcond0} are replaced by \eqref{posL} and \eqref{psdcond2}, respectively.
% \end{corollary}
See Appendix~\ref{app:stability} for the detailed proof.

Figure~\ref{fig:SRNprocess_stableAndUnstable} shows an example of stochastic processes generated from a stable recurrent network  and an unstable recurrent network  respectively. In the upper panel, the recurrent network  system parameters satisfy the stability condition \eqref{psdcond0}, while those in the lower panel violate  \eqref{psdcond0}. (In both situations, the number of nodes is 10 and the diffusion parameter $\sigma$ is fixed at  $0.5$.)
The differences between the two models are obvious.

In reality,  asymptotically stable systems are commonly observed. The stability conditions reflect  structural characteristics. For example, when all $l_i$ are equal and $d_i>0$, then the \textit{skew-symmetry} of $\bsbA$, i.e., $\bsbA = - \bsbA^\tran$, guarantees asymptotic stability. The  information provided by the constrains can assist topology learning. This  motivates the design of sparse recurrent network learning with stability.
\begin{figure}[h!]
  \centering
    \includegraphics[width=0.5\textwidth]{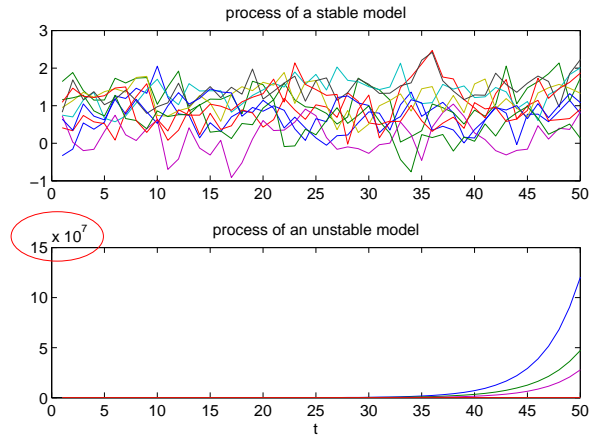}
 \caption{Stochastic processes generated from a stable recurrent network model and an unstable recurrent network model.}
\label{fig:SRNprocess_stableAndUnstable}
\end{figure}

\subsection{{\sss} network  estimation }
\label{sec:s3}

\eqref{psdcond0} is less restrictive than  \eqref{psdcond1}.
In optimization, imposing \eqref{psdcond0} seems however difficult. %One can turn to $M$-matrix conditions or Gershgorin disk theorem. Assuming no self-regulation, one can impose the linear constraints $l_i \|\bsba_i \|_1\leq 4 d_i$ for each node.
We propose
\begin{align}
\begin{gathered}
\min_{\bsbB, \bsbL, \bsbD, \bsbu, \bsbc}  \frac{1}{2} \| \bsbW^{1/2} \{ \bsbY - [\bsb{\pi}(\bsbX \bsbB + \bsb{1} \bsbu^\tran) \bsbL - \bsbX \bsbD + \bsb{1} \bsbc^{\tran} ]\} \|_F^2  + P(\bsbB,\bsbLamb)=:F_1, \\\mbox{ s.t. }  \bsbL \succeq \bsb{0}, \bsbD  \succeq \bsb{0},  (\bsbL   \bsbB^\tran + \bsbB  \bsbL)/2   \preceq  4 \bsbD,
\end{gathered}
\label{eq:multis3_obj}\end{align}
referred to as the Stable-Sparse Sigmoidal (\sss) network learning.
The stability constraints are now imposed on $\bsbB$ as the transpose of the raw coefficient matrix $\bsbA$. Similar to the discussion of \eqref{eq:constraintForm},  in implementation, we  add mild $\ell_2$ penalties on $l_i$ and $c_i$ to deal with possible collinearity, and it is  common to  replace $\bsb{0}$ by $\epsilon\bsbI$ with $\epsilon$ extremely small.

In this part, we focus on the $\ell_1$ penalty  $P(\bsbB,\bsbLamb)=\| \bsbLamb \circ \bsbB\|_1$ where matrix $\bsb{\Lambda}$  usually has the form $\bsb{\Lambda}[k, k']=\| \bsbX[:, k]\|_2$ for any $k,k'\leq n$ (other options are possible.) Introducing the componentwise regularization matrix is helpful  when one wants to  maintain or  forbid certain links based on prior knowledge or preliminary screening.
For example, if no self-regulation is allowed, then all diagonal entries of $\bsb{\Lambda}$ ought to be  $+\infty$.

It turns out that one can modify Step 2 and Step 4 of the RNS algorithm to solve \eqref{eq:multis3_obj}.

First,  the non-negativity of $l_i$ and $d_i$ can be directly incorporated, if  the weighted least-squares problem in Step 2  is replaced by the following  programming problem with generalized non-negativity constraints:
\begin{align}
\mbox{\textit{\textbf{Step 2)}}} &  \qquad  \mbox{ Solve }
%\begin{gathered}
\min_{\bsbL,\bsbD,\bsbc} \|\bsbW^{1/2}\{\bsbY-[\bsbmu^{(j)} \bsbL - \bsbX \bsbD + \bsb{1} \bsbc^\tran]\}\|_F^2 \notag \\
\mbox{s.t. } & \bsbL =\mbox{diag}\{l_i\} \succeq \bsb{0}, \bsbD =\mbox{diag}\{d_i\} \succeq \bsb{0},  (\bsbL   \bsbB^{(j-1) \tran}+ \bsbB^{(j-1)}  \bsbL)/2   \preceq  4 \bsbD.
%\end{gathered}
\label{objstep2}
\end{align}
A variety of algorithms and packages can be used. % \cite{gill1981practical}.
(The technique in Appendix \ref{app:s3conv} also applies.)

Integrating  the spectral constraint on $\bsbB$ into network learning is much trickier. We modify Step 4 of Algorithm \ref{alg:sigscr} as follows.
\\

\begin{compactenum}
\item[\textbf{\textit{Step 4)} }] Update $\bsbB$ and $\bsbu$:
\begin{compactenum}
\item[\textbf{\textit{4.1)}}]  $\tilde{\bsb{\Xi}} \gets  \tilde\bsbB^{(j-1)} -  \tilde\bsbX^{\tran} \{\tilde\bsbW \circ \bsbxi(\tilde\bsbX \tilde\bsbB^{(j-1)}, \tilde\bsbY)\} (\bsbK^{(j)})^{-}$,  $\bsbu^{(j)}\gets (\tilde{\bsb{\Xi}}[1,:])^\tran$, ${\bsb{\Xi}}\gets \tilde{\bsb{\Xi}}[2:\mbox{end}, :]$, and  $\bsb{\Lambda}^{(j)} = \bsb{\Lambda}\cdot  (\bsbK^{(j)})^{-}$.
\item[\textbf{\textit{4.2)}}]  Perform the inner loop iterations, starting with
$\bsbB_3 \gets \bsb{\Xi}$, $\bsbC_3 \gets (\bsbL^{(j)}   \bsbB_3^\tran + \bsbB_3  \bsbL^{(j)})/2$, $\bsbP = \bsb{0}$, $\bsbQ_{\bsbB}=\bsb{0}$, $\bsbQ_{\bsbC}=\bsb{0}$, $\bsbR=\bsb{0}$, and the operators ${\mathcal P}^1, {\mathcal P}^2, {\mathcal P}^3$ defined in Lemmas \ref{proj_l1}-\ref{proj_spec}:\\
%\WHILE{not converged}
\textbf{repeat}
\begin{compactenum}[i)]
\item $\bsbB_1 \gets {\mathcal P}^1 ( \bsbB_3 + \bsbP; \bsb{\Lambda}^{(j)})$, $\bsbC_1 \gets \bsbC_3$, $\bsbP \gets \bsbP + \bsbB_3 - \bsbB_1$.
\item
  $[\bsbB_2, \bsbC_2] \gets {\mathcal P}^2(\bsbB_1 + \bsbQ_{\bsbB}, \bsbC_1 + \bsbQ_{\bsbC}; \bsbL^{(j)}) $, \\ $\bsbQ_{\bsbB} \gets \bsbQ_{\bsbB}+\bsbB_1- \bsbB_2$, $\bsbQ_{\bsbC}\gets  \bsbQ_{\bsbC} +  \bsbC_1 - \bsbC_2$.  % $\bsbQ_{\bsbC} \gets \bsbQ_{\bsbC} + \bsbC_1 - \bsbC_2$
\item  $\bsbB_3 \gets \bsbB_2$, $\bsbC_3 \gets {\mathcal P}^3(\bsbC_2 + \bsbR; \bsbD^{(j)})$, $\bsbR\gets \bsbR + \bsbC_2 - \bsbC_3$.
\end{compactenum}
\textbf{until} convergence
\\
$\bsbB^{(j)} \gets \bsbB_3$
\item[\textit{\textbf{4.3}) }] $\tilde\bsbB^{(j)}\gets [\bsbu^{(j)} \ (\bsbB^{(j)})^\tran]^\tran$
\end{compactenum}
\end{compactenum}

Algorithm  \ref{alg:sigscr} with such modifications in  Step 2 and Step 4 is referred to as the {\sss}  estimation algorithm.\\

\begin{theorem} \label{th:s3conv}
Given any initial point $\tilde\bsbB^{(0)}$, the  {\sss}  algorithm
converges in the sense that the function value decreasing property holds,
% $$
% F_1(\bsbB^{(j-1)},\bsbL^{(j-1)},\bsbD^{(j-1)}, \bsbu^{(j-1)}, \bsbc^{(j-1)}) \geq F_1(\bsbB^{(j)},\bsbL^{(j)},\bsbD^{(j)}, \bsbu^{(j)}, \bsbc^{(j)}),
% $$
and furthermore, $\bsbB^{(j)}$, $\bsbL^{(j)}$, and $\bsbD^{(j)}$ satisfy $\bsbL^{(j)} \succeq \bsb{0}, \bsbD^{(j)}  \succeq \bsb{0},  (\bsbL^{(j)}   (\bsbB^{(j)})^\tran + \bsbB^{(j)}  \bsbL^{(j)})/2   \preceq  4 \bsbD^{(j)}$ for any $j\geq 1$.
\end{theorem}

The proof is given in Appendix \ref{app:s3conv}.

We observe that practically, the inner loop converges fast (usually within 100  steps). Moreover, to guarantee the functional value is decreasing, one  does not have to run the inner loop till convergence. %; typically [???] steps suffice. %
Although it is possible to apply the stable-sparse estimation  directly,  we recommend running the screening algorithm (RNS) first, followed by the fine {\sss} network learning.

%\textbf{Manner 2.}
%Although it is possible to apply the stable-sparse estimation  directly,  we recommend running the screening algorithm (RNS) first followed by the fine {\sss} network learning. In this screening manner, the value of  $m$ in RNS represents an upper bound of the true cardinality.  In the second stage, to maintain the zeros obtained from  RNS, it suffices to  set the associated $\lambda_{ij}$'s in \eqref{eq:multis3_obj}  to be $\infty$. Then a parameter tuning with respect to $\bsb{\Lambda}$ is necessary. We may choose one according to BIC, or desired cardinality.
%
%\textbf{Manner 1.} Another usage  is to use {\sss} for stability correction: given the RNS estimate $\hat\bsbB$, set $\lambda_{ij}=\infty$ if $\hat b_{ij}=0$ and $0$ otherwise. In this way, the {\sss} does not  enforce further sparsity.
%
%\textbf{Manner 3.} Finally, we can integrate RNS and {\sss} in a smart way to gain both efficiency and stability. The idea is similar to Manner 2, but we generate a solution path not according to a pre-specified grid of $\lambda$, but directly based on the cardinality control parameter. Specifically, given a candidate cardinality $m$, run RNS to get a sparse estimate; because $\lambda$ has the meaning of  the threshold of $\bsb{\Xi}$, we fix $\lambda$ at a value based on   $\bsb{\Xi}$ and $\bsbK$ from RNS.
%Repeat the procedure for each $m$.
%Finally, the last-stage tuning can be made by BIC or cardinality control (which is similar to Manner 2).

\section{Experiments}
\label{sec:dataexps}
\subsection{Simulation Studies}
\label{sec:simulation}
In this subsection, we conduct synthetic data experiments to demonstrate the performance of the proposed learning framework in  recurrent network analysis. %\textbf{$\bullet$ Comments:} Some notations are not consistent with the previous parts of the manuscript. Try to show
%\begin{itemize}
%\item better than linear model
%\item $\ell_1$ does not work well
%\item stability
%\end{itemize}
%\subsection{Model simulation}
%\label{sec:modelSimulation}
An Erd\H{o}s-R\'enyi-like scheme of generating  system parameters, including a sparse regulation matrix $\bsbA$, is described as follows. Given any node $i$, the number of its regulators is   drawn from a binomial distribution $\mathcal{B}(n-1,\frac{1}{2n})$. The regulator set  $S_i$ is  chosen randomly from the rest $(n-1)$ nodes (excluding node $i$ itself). If $j \notin S_i$,  $a_{ij}=0$. Otherwise,  $a_{ij}$ follows  a mixture of two Gaussians  $\mathcal{N}(1.5,0.1^2)$ ad $\mathcal{N}(-1.5,0.1^2)$ with probability $1/2$ for each. Then draw random $\bsbl$, $\bsbu$, $\bsbc$  from Gaussian distributions (independently)   $l_i\sim\mathcal{N}(1.5,0.1^2), u_i\sim\mathcal{N}(0,0.1^{2}), c_i\sim\mathcal{N}(0,0.1^{2})$. Finally   $\bsb{d}$ is generated so that the system satisfies the stability condition~\eqref{psdcond0}. %[\textbf{Yuejia: some parameters might better be adjusted, such as the magnitude of $a_{ij}$ (and $l_i$ should be larger?}]

\textit{Topology identification.}
\label{sec:RNSperformance}
%\textbf{In this part, you need to compare RNS and TSNI (not mentioning its linearity but its popularity). No need for three examples. Only plot the ROC curves. Mention  RNS facilitate parameter tuning (with no need of the figures).}
First, we  test the performance of   RNS in  recurrent network topology identification. We  compare it with    TSNI \cite{bansal2006sos} and  QTIS  \cite{S2}.   TSNI is a popular network learning approach and applies  principle component analysis for dimensionality reduction.  QTIS is a network screening algorithm based on sparsity-inducing techniques.
To avoid the ad-hoc issue of parameter tuning,  ROC curves will be used to summarize  link detection results in a comprehensive way, in terms of true positive rate (TPR) and false positive rate (FPR).

%RNS produces inherently sparse solutions, is able to achieve better identification performance. It provides the user a direct and easy control over the network cardinality.

We simulate two networks  according to the scheme introduced early. In  the first example we set   $n=10, T=100$,
% and the number of truly existing links (denoted by $m^*$) is  $m^*=13$
 in the second  example  $n=200, T=500$, and in the third $n=100, T=1000$. % and  $m^*=30$.
Given all system parameters, we can call  Matlab functions {\sc sde} and {\sc sde.simulate} to generate   continuous-time stochastic processes  according to \eqref{eq:sigsde}.
%To approximate the continuous stochastic model accurately enough, we choose $\rd t=0.001$.
The discrete-time observations are sampled from the stochastic processes with sampling period $\Delta T = 1$.

In this experiment,    the number of unknowns  in either case  is larger than the sample size, especially in {Ex.2} which has  about \textbf{41K} variables but only 500 observations.    Given any algorithm, we vary the target cardinality parameter $m$ from 1 to $n(n-1)$, collect all estimates,  and compute their associated TPRs and FPRs. The experiment is repeated for $50$ times and the averaged  rates are shown in  the  ROC curves in   Figure \ref{fig:roc}.
RNS beats TSNI and QTIS  by  a large margin. In fact, the  ROC curve of RNS  is \textit{everywhere} above the TSNI curve  and the QTIS curve.
%This shows the  importance of   taking the nonlinearity of the recurrent network model into account and the power of inherent sparsity promoted learning.
\begin{figure*}[!]
  \centering
  \subfloat[Example 1.]{\label{fig:ex1_roc}\includegraphics[width=0.33\textwidth]{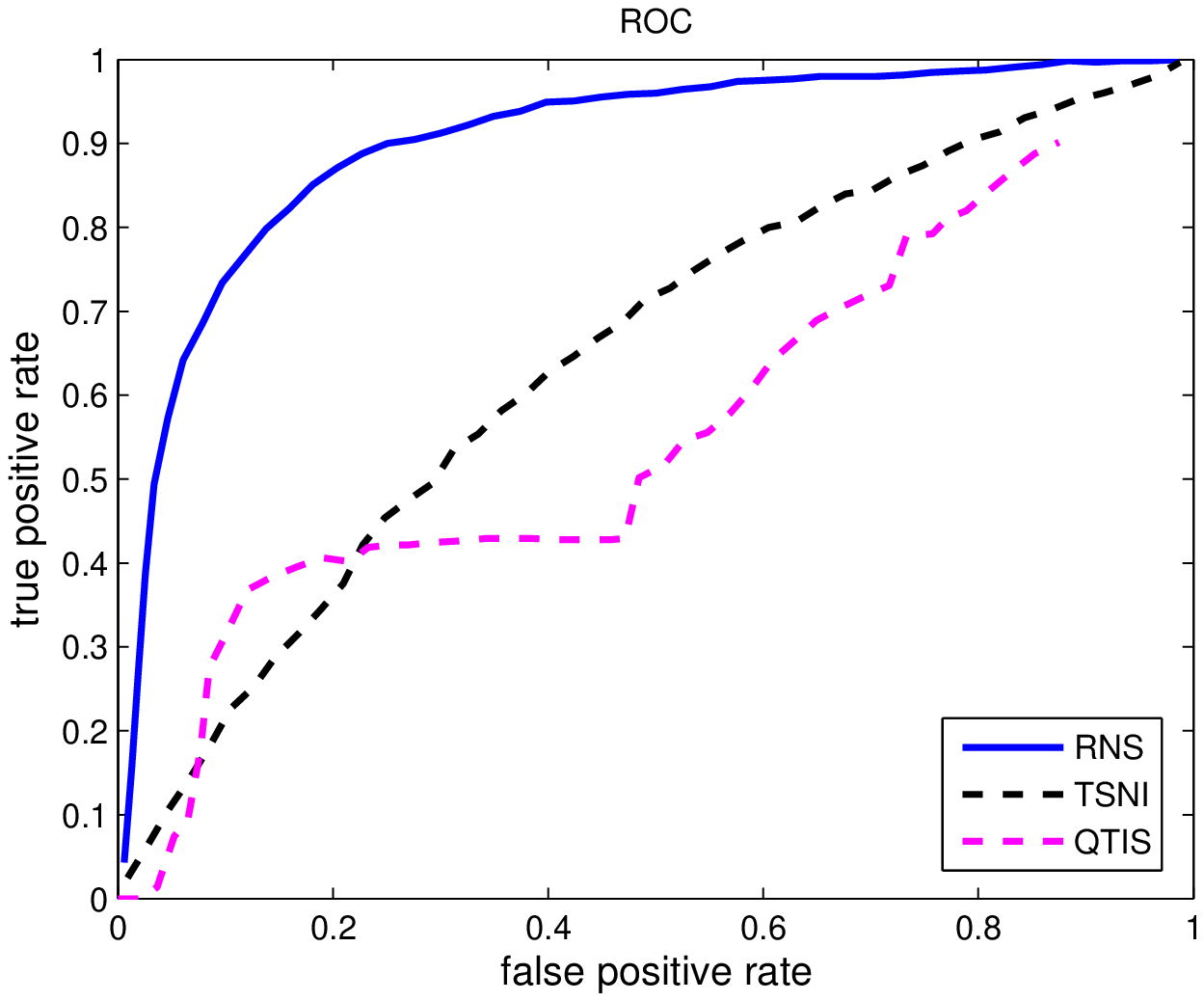}}
  \subfloat[Example 2.]{\label{fig:ex2_roc}\includegraphics[width=0.33\textwidth]{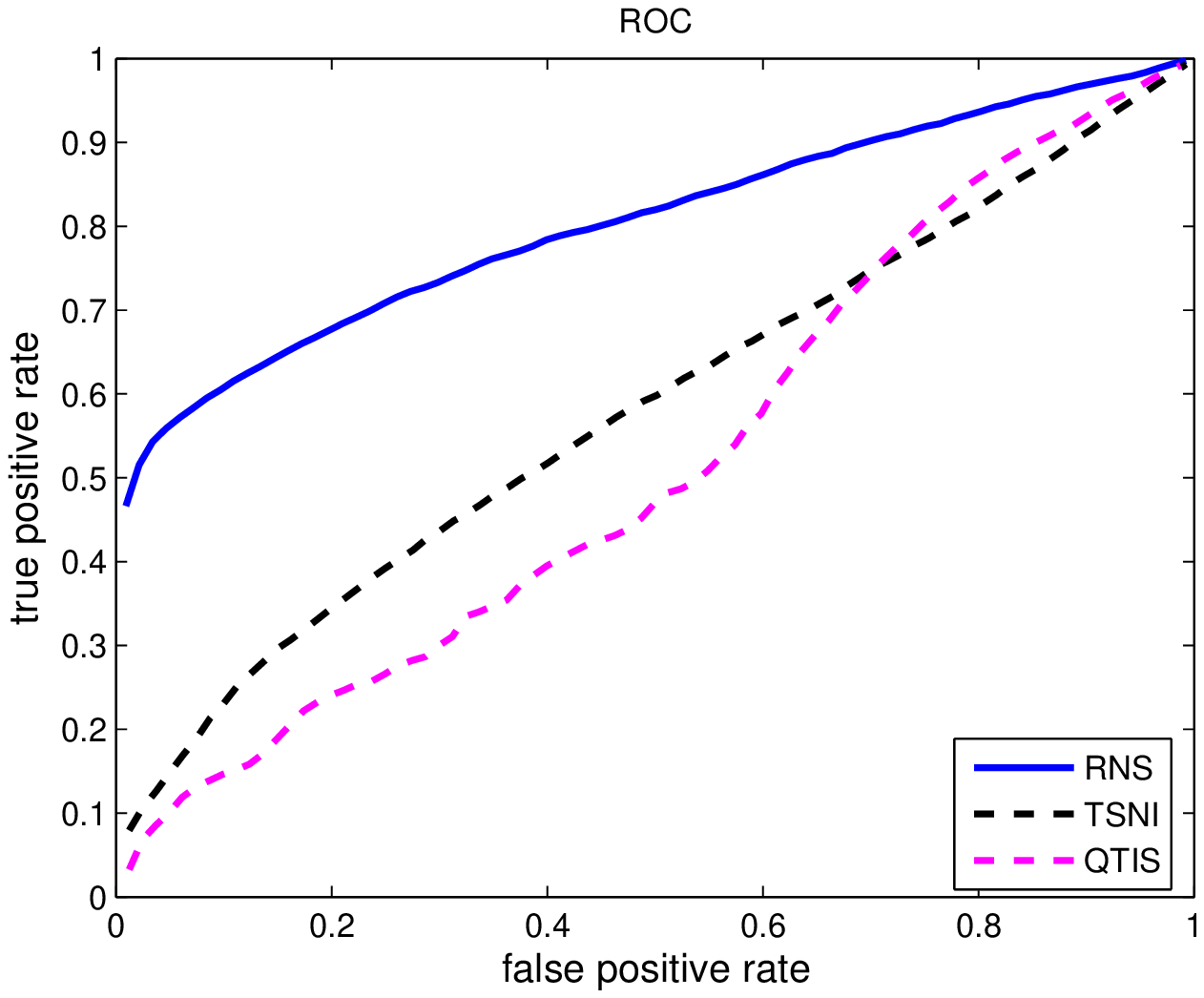}}
  \subfloat[Example 3.]{\label{fig:ex3_roc}\includegraphics[width=0.33\textwidth]{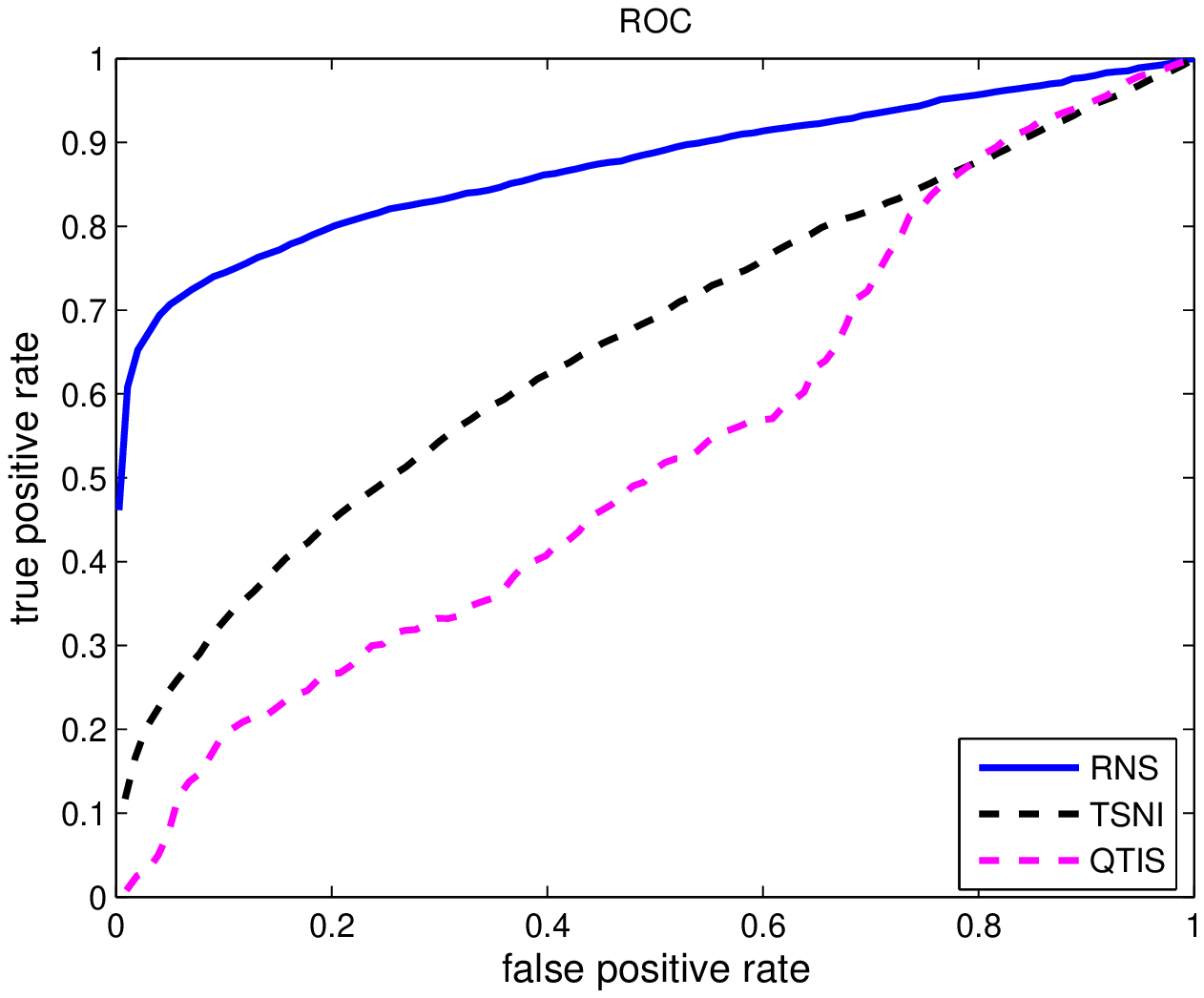}}
  \caption{Comparison of ROC curves.  \label{fig:roc}}
\end{figure*}

\textit{System stability.}
\label{sec:S3performance}
%\textbf{In this part, we compare \sss, L1-sig, and L1-lin, all of which use the $\ell_1$ penalty. Table 2 is kept (while Table 1 is not necessary). No need to very many $h$ values. 3 or 4 showing the trend and advantage is enough. Ideally, when $h$ is larger, {\sss} becomes even more powerful. Use $h=1$ to argue the {\sss} integrates structural information and is parsimonious.}
Next, we  show the necessity of  stable learning  in network dynamics forecasting. For simplicity,  the $\ell_1$ penalization is used. We compare   the {\sss} network estimation with the  approach based on sparse sigmoidal regression  in Section~\ref{sec:univariate} (with  no stability guarantee), denoted by {SigSpar}.

We  use two network examples (Ex.4 and Ex.5) with $n=20, 40$ respectively. In each setting, we generate    $T=20$  samples for training, and $200$ validation samples for  parameter tuning.  In this experiment, forecast error at  a future time point is the major  concern.  Suppose  $\bsbx(T)$, the network snapshot at $T$,  is  given.  With  system parameter estimates obtained, one can simulate a stochastic  process $\bsb{\hat{x}}(t)$ ($t\geq T$) starting with $\bsbx(T)$ based on model \eqref{eq:sigsde}. The forecast error at time point $T+h$ is defined as $\mbox{FE}=\|\bsbx(T+h) - \bsb{\hat{x}}(T+h)\|_2^2/n$. Long-term forecasting  corresponds to  large values of   $h$.   We repeat the experiment for  50 times and show the average FE  in Table~\ref{tab:forecasting}. The error of  SigSpar becomes extremely large as $h$ increases, because there is no stability guarantee of its network estimate, while {\sss} has excellent performance even in long-term forecasting.
\begin{table*}
\caption{Forecasting error comparison on the simulation data.} \label{tab:forecasting}
\begin{center}
\begin{tabular}{c|crrrrr}
  \hline\hline
      &  $h$ & 1 & 5  & 10 & 15 & 20
  \\ \hline
  \multirow{2}{*}{Ex.4}
 &   SigSpar  &   0.08  & 3.01   &  41.95  &  450.54   & 4.4$\times$10$^3$ \\
      & \sss &   0.04  & 0.48   &  1.01  &   1.39  &  1.76 \\
\hline \hline
  \multirow{2}{*}{Ex.5}
 &   SigSpar  &   0.35 & 23.76   & 41.95  &  7.3$\times$10$^3$ & 1.0$\times$10$^5$ \\
      & \sss &   0.06  & 0.70   &  1.75  &  2.80   &  4.04 \\
\hline \hline
\end{tabular}
\end{center}
\end{table*}

\subsection{Real data}
\label{sec:application}

%\textbf{Give a gene example and a neuro example.}

%The recurrent network model can be applied to capture many nonlinear dynamical networks in a variety of areas. In this section, we show some interesting applications of the {\sss} learning.

%\begin{table}
%\begin{center}
%\begin{tabular}{c|crrr}
%  \hline
%      &  $t_s$ (min) & 7 &  14  & 21 \\ \hline
%  \multirow{2}{*}{$m=100$}
% &   plain RNS  &   55.5  & 160.6   &  137.7  \\
% & \textit{RNS} &    6.8  &  22.9   &   30.1  \\
%\hline \hline
%  \multirow{2}{*}{$m=200$}
% &   plain RNS  &    30.3 & 85.8   &   66.0  \\
% & \textit{RNS} &    6.7  & 22.9   &   30.3  \\
%\hline \hline
%  \multirow{2}{*}{$m=300$}
% &   plain RNS  &   22.4  & 73.2  & 55.0  \\
% & \textit{RNS} &   6.7   & 22.9  & 30.3   \\
%\hline \hline
%\end{tabular}
%\caption{Forecasting performance for the Yeast cell cycle subnetwork.} \label{tab:forecasting_yeast}
%\end{center}
%\end{table}

\textit{Yeast gene regulatory network.}
%\textbf{We use the Yeast network to demonstrate an application. Mention other results are avaialbe for say SOS, but not shown here due to space limitations. In Table 4, don't forget to argue the performance is robust to the choice of $m$}
We use RNS to study the transcriptional regulatory  network in   the yeast cell cycle.  The  dataset is publicly available and a detailed description of the microarray experiment is in \cite{spellman1998yeast}.
Following \cite{Chen2004}, we focus on the  20 genes whose expression patterns fluctuate in a periodic manner. The dataset contains their  expression levels recorded at 18 time points during a cell cycle. %The sampling period is 7 min.
In this regulatory (sub)network, five genes have been identified as transcription factors, namely SWI4, HCM1, NDD1, SWI5, ASH1, and 19 regulatory connections from them to  target genes have been reported with direct evidence in the literature (cf. the YEASTRACT database at \url{http://yeastract.com/}). \cite{Chen2004} found 55 connections from the transcription factors to the target genes, of which 14 have  biological evidence (and so the true positive rate is   $14/19=73.7\%$). For a fair comparison, we also let RNS detect 55 connections from the transcription factors to the target genes, and achieved a higher  true positive rate  $89.5\%$. A detailed comparison of the identified  regulatory connections is shown in Table~\ref{tab:yeast}.

\begin{table*}
\caption{Regulatory connections of the Yeast cell cycle subnetwork. `$\times$' stands for the regulatory connections  identified by RNS,  `$\underline{\ }$' stands for those  by \cite{Chen2004}, and  `$\square$' stands for the confirmed regulatory connections  with published evidence in the literature.  }
\label{tab:yeast}
\centering
\begin{tabular}{cccccc} \hline
      &      SWI4      &    HCM1      &     NDD1      &    SWI5       &     ASH1   \\ \hline
SWI4  &                &              &               &            &       \\
HCM1  &   $\underline{\boxtimes}$    &              &               &            &       \\
NDD1  &   $\underline{\boxtimes}$   &              &               &            &       \\
SWI5  &                &             &               &            &       \\
ASH1  &                &              &               &   $\underline{\boxtimes}$  &       \\
CLN2  &   $\underline{\boxtimes}$  &              &               &            &       \\
SVS1  &   $\underline{\boxtimes}$  &              &               &            &       \\
SWE1  &   $\underline{\boxtimes}$   &              &               &            &       \\
MNN1  &   $\underline{\boxtimes}$ &              &               &           &       \\
CLB6  &   $\underline{\boxtimes}$ &             &               &           &       \\
HTA1  &   $\underline{\boxtimes}$  &              &               &            &       \\
HTB1  &   $\underline{\boxtimes}$  &              &               &            &     \\
HHT1  &               &              &               &            &       \\
CLB4  &                &              &               &  $\boxtimes$ &       \\
CLB2  &   $\boxtimes$   &              &    $\underline{\boxtimes}$  &            &       \\
CDC20 &               &              &               &           &       \\
SPO12 &                 &              &   $\underline{\square}$  &            &       \\
SIC1  &                 &              &               &  $\underline{\boxtimes}$  &       \\
CLN3  &    $\boxtimes$   &   $\boxtimes$    &            & $\underline{\boxtimes}$& $\square$\\
CDC46 &                  &              &              &            &       \\
\hline
\end{tabular}
\end{table*}

%We have also successfully applied the RNS algorithm to other gene networks, such as the SOS network in the bacteria E. coli \cite{bansal2006sos}. The detailed results are not shown due to space limitations.

\textit{fMRI data.}
%\textbf{This example is necessary.}
The resting state fMRI  data  provided  by the ADHD-200 consortium \cite{milham2012adhd}  have been preprocessed by the NIAK interface \cite{lavoie2012integration}.  The dataset we are using has 954 ROIs (regions of interest) and 257 time points. In the experiment, the first 200 observations are for training (and so $T=200$ and $n=954$), and the following 57 observations are for testing.
We applied  RNS  to get a network pattern, followed by  the {\sss} network estimation for stability correction.
%The testing error is compared with that of TSNI. Denote $\bsbX, \bsbY$ as the testing data. For {\sss} learning, the testing error is defined as $ TE_{sigmoidal} =  \|  \bsbY_t - [\bsbpi(\bsbX_t \bsb{\hat{B}} + \bsb{1} \hat\bsbu^\tran) \hat\bsbL - \bsbX_t \hat\bsbD + \bsb{1} \hat\bsbc^{\tran} ] \|_F^2$; for TSNI, the testing error is defined as $TE_{linear} = \|\bsbY_t - \bsbX_t \bsb{\hat{B}}\|_F^2$.
  Table~\ref{tab:test_fMRI} shows the results  averaged over 10 randomly chosen subjects.
 Our learning algorithm  can achieve much lower error than TSNI, and  its performance is pretty robust to the choice of $m$.
\begin{table*}
\caption{Forecast error  on the fMRI data.} \label{tab:test_fMRI}
\begin{center}
\begin{tabular}{crrrrr}
  \hline
$m/n^2$ &   0.5    &  0.6  &  0.7   &  0.8   &  0.9
  \\ \hline
 TSNI  &   0.981  & 1.021  &  1.044  &  1.056  & 1.060 \\
  \sss &   0.194  & 0.194  &  0.194  &  0.194  &  0.194 \\
\hline
\end{tabular}
\end{center}
\end{table*}

%\newpage
%\appendix
\appendices
\section{Proof of Theorem~\ref{th:fValDecrease}}
\label{app:fValDecrease}
For notation simplicity,  we introduce  $\mu_s(\bsbb)  := \pi(\tilde\bsbx_s^{\tran} \bsbb)$. Then
$\partial \mu_s(\bsbb)/\partial \bsbb = \mu_s(\bsbb) (1-\mu_s(\bsbb)) \tilde\bsbx_s^{\tran}$,
from which it follows that
\begin{align*}
\nabla  ( \sum_{s=1}^T w_s (y_s-\mu_s(\bsbb))^2/2 )=\tilde\bsbX^{\tran} \bsbW\bsbxi(\tilde\bsbX \bsbb, \bsby),
\end{align*}
where $\bsbW = \mbox{diag}\{w_s\}_{s=1}^T$ and $\bsbxi$ is the function defined in \eqref{xidef}  applied  componentwise.
Moreover, we can compute its Hessian (details omitted)
\begin{align*}
\bsbH & =D(\nabla(\sum w_s (y_s-\mu_s(\bsbb))^2/2))\\
&=
\tilde\bsbX^{\tran} \bsbW \mbox{diag}\left\{ \mu_s(1-\mu_s)\right.\left.[\mu_s(1-\mu_s) + (\mu_s-y_s)(1-2\mu_s)]\right\}  \tilde\bsbX\\
&=:
\tilde\bsbX^{\tran}  \bsbW \bsbSig(\bsbb) \tilde\bsbX.
\end{align*}
%where $\bsbSig(\bsbb)\eqdef  \mbox{diag}\left\{ \mu_s (\bsbb) (1-\mu_s(\bsbb))[\mu_s(\bsbb)(1-\mu_s(\bsbb)) + (\mu_s(\bsbb)-y_s)(1-2\mu_s(\bsbb))]\right\}$.
Note that  $\bsbH$ is not necessarily positive  semi-definite.

Let  $F_0(\bsbb) = \frac{1}{2}\sum w_s (y_s-\mu_s(\bsbb))^2  + \sum P(\beta_k; \lambda)$. Define a   surrogate function as
$
G(\bsbb, \bsbb') =  \frac{1}{2}\sum w_s (y_s-\mu_s(\bsbb'))^2  + \sum P(\beta_k'; \lambda) +\frac{ K }{2 }\|\bsbb-\bsbb'\|_2^2
+ \frac{1}{2}\sum w_s ((y_s-\mu_s(\bsbb))^2 - (y_s-\mu_s(\bsbb'))^2)   + \sum w_s \xi(\bsbx_s^{\tran}\bsbb, y_s) (\bsbx_s^{\tran}\bsbb' - \bsbx_s^{\tran}\bsbb).$
 Based on the previous calculation, we have
\begin{align*}
%\label{eq:sigpro_mid}
&\frac{1}{2}\sum w_s ((y_s-\mu_s(\bsbb))^2 - (y_s-\mu_s(\bsbb'))^2) \\
=&(\tilde\bsbX^\tran\bsbW\bsbxi(\tilde\bsbX\bsbb,\bsby))^\tran(\bsbb-\bsbb') -\frac{1}{2}(\bsbb-\bsbb')^\tran(\tilde\bsbX^\tran\bsbW\bsbSig(\theta\bsbb+(1-\theta)\bsbb')\tilde\bsbX)(\bsbb-\bsbb'),
\end{align*}
for some $\theta \in [0, 1]$. Let $\bsb{\zeta}=\theta\bsbb+(1-\theta)\bsbb'$. It follows that
\begin{align*}
& \frac{ K }{2 }\|\bsbb-\bsbb'\|_2^2 + \frac{1}{2}\sum w_s ((y_s-\mu_s(\bsbb))^2 - (y_s-\mu_s(\bsbb'))^2)   \\
  & \quad + \sum w_s \xi(\tilde\bsbx_s^{\tran}\bsbb, y_s) (\tilde\bsbx_s^{\tran}\bsbb' - \tilde\bsbx_s^{\tran}\bsbb) \\
 = & \frac{1}{2}(\bsbb'-\bsbb)^\tran (K \bsbI - \tilde\bsbX^{\tran} \bsbSig(\bsb{\zeta}) \bsbW \tilde\bsbX) (\bsbb'-\bsbb)\\   %  \quad \mbox{ for some } \theta \in (0, 1)
 \geq & \frac{K - \| \tilde\bsbX\|_2^2 \| \bsbSig(\bsb{\zeta}) \bsbW \|_2}{2} \| \bsbb-\bsbb'\|_2^2.
\end{align*}
Because
\begin{equation}
\begin{aligned}
&w_s\mu_s(1-\mu_s)((1-2\mu_s)(\mu_s-y_s)+\mu_s(1-\mu_s))\\
\leq & \frac{w_s}{4}\left(\frac{1}{2}\left(\frac{1-2\mu_s+2\mu_s-2y_s}{2}\right)^2 +\frac{1}{4}\right )\\
= & \frac{w_s}{4}\left( \frac{(1-2y_s)^2}{8} +\frac{1}{4} \right),
\end{aligned} \nonumber
\end{equation}
the diagonal entries of $\bsbW\bsbSig(\bsb{\zeta})$ are uniformly bounded by
$$
\max_s \frac{w_s}{4}\left( \frac{(1-2y_s)^2}{8} +\frac{1}{4} \right)
$$
or $k_0(\bsby, \bsbw)$ (see \eqref{k0def}).
Therefore, choosing $K\geq k_0(\bsby, \bsbw) \|\tilde\bsbX\|_2^2$, we have
$
G(\bsbb, \bsbb') = \frac{1}{2}\sum w_s (y_s-\mu_s(\bsbb'))^2  + \sum P(\beta_k'; \lambda)  +\frac{1}{2}(K-\|\tilde\bsbX\|_2^2\|\bsbW\bsbSig\|_2^2)
\geq F_0(\bsbb)
$
for any $\bsbb, \bsbb'$. %When,  The equality holds if and only if $\bsbb=\bsbb'$.
%Since the diagonal entries of $\bsbD \bsbW$ are uniformly bounded by
%$$\max_s \frac{ w_s}{4} \left(\frac{1}{4} + \frac{1}{2} \left(\frac{2\mu_s - 2y_s + 1 - 2\mu_s}{2}\right)^2\right),$$
%if we choose $K\geq k_0(\bsby, \bsbw) \|\bsbX\|_2^2$, $G(\bsbb, \bsbb') \geq F_0(\bsbb')$ for any $\bsbb, \bsbb'$.

On the other hand, based on the definition, it is easy to show (details omitted) that given $\bsbb$,  $\min_{\bsbb'} G(\bsbb, \bsbb')$ is equivalent to:
\begin{align}
\min_{\bsbb'} \frac{ K }{2 }\|\bsbb'-\bsbb + \frac{1}{K}\tilde\bsbX^{\tran} \bsbW \bsbxi(\tilde\bsbX \bsbb, \bsby) \|_2^2 + \sum  P(\beta_k'; \lambda). \label{eq:subopt}
\end{align}
%\begin{align}
%\label{eq:orthogonal}
%\min_{\bsbb'} \frac{ K }{2 }\left\|\bsbb'-\bsbb + \frac{1}{K}\bsbX^{\tran} \bsbW \bsbxi(\bsbX \bsbb, \bsby)\right \|_2^2 + \sum P(\beta_k'; \lambda)
%\end{align}
%or
%\begin{align*}
%\min_{\bsbb'} \frac{ 1 }{2 }\left\|\bsbb'-\bsbb + \frac{1}{K}\bsbX^{\tran} \bsbW \bsbxi(\bsbX \bsbb, \bsby)\right \|_2^2 + \sum \frac{1}{K} P(\beta_k'; \lambda). \label{subopt}
%\end{align*}
%Assume $P$ is the $l_1$-penalty \eqref{l1-pen}. Then
Applying Lemma 1 in \cite{SheGLMTISP} without requiring uniqueness, there exists a globally optimal solution
\begin{align}
\bsbb_{o}'(\bsbb)=\Theta(\bsbb - \tilde\bsbX^{\tran} \bsbW \bsbxi(\tilde\bsbX \bsbb, \bsby)/K; \lambda'), \label{orthsol}
\end{align}
provided that  $P(t;\lambda')=P(t;\lambda)/K$ for any $t$.
%Because $\Theta$ is monotone, the  set of discontinuities is at most countable, and so \eqref{orthsol} holds with probability 1.
In summary, we obtain $F_0(\bsbb) = G(\bsbb, \bsbb ) \geq G(\bsbb, \bsbb_{o}'(\bsbb) ) \geq F_0(\bsbb_{o}'(\bsbb))$.

We are now in a position  to prove \eqref{funcvaldes}.  In fact,
given    $\bsbg$ and $l$, the optimization problem
$$
\min_{\bsbb}\frac{1}{2}\sum  w_s (y_s -\bsbz_s^\tran \bsbg- l \pi(\tilde\bsbx_s^{\tran} \bsbb))^2 + \sum P(\beta_k, \lambda),
$$
is equivalent to
$\min_{\bsbb}\frac{1}{2}\sum  w_s' (\tilde y_s -  \pi(\tilde\bsbx_s^{\tran} \bsbb))^2 +  \sum P(\beta_k, \lambda)
$
with $\tilde y_s = (y_s - \bsbz_s^\tran \bsbg)/l$, $w_s' = l^2 w_s$. (Note that the $l$ obtained from WLS is nonzero with probability 1.) Therefore, the function value decreasing property always holds during the iteration.

\section{Proof of Theorem~\ref{th:constraintform}}
\label{app:constraintform}
Define a quantile thresholding rule $\Theta^\#(\cdot;m,\eta)$ as a variant of the hard-ridge thresholding rule \eqref{eq:l02-pen}. Given $1\leq m\leq np$: $\bsb{A}\in \mathcal{R}^{n\times p} \rightarrow \bsbB\in \mathcal{R}^{n\times p}$ is defined as follows: $b_{ij}=a_{ij}/(1+\eta)$ if $|a_{ij}|$ is among the $m$ largest in the set of $\{|a_{ij}|: 1\leq i\leq n, 1\leq j\leq p\}$, and $b_{ij}=0$ otherwise.

\begin{lemma}
\label{lemmaQuantile}
$\bsb{\hat{B}}=\Theta^\#(\bsb{A};m,\eta)$ is a globally optimal solution to
\begin{equation} \begin{gathered}
\min_{\bsbB} f_0(\bsbA) = \frac{1}{2} \|\bsb{A}-\bsbB\|_F^2 + \frac{\eta}{2} \|\bsbB\|_F^2 \\
\mbox{s.t. } \|\bsbB\|_0 \leq m.   \nonumber
\end{gathered}
\end{equation}
\end{lemma}

\begin{proof}
Let $I \subset \{(i,j)| 1\leq i\leq n, 1\leq j\leq p\}$ with $|I|=m$. Assuming $\bsbB_{I^c}=\bsb{0}$, we get the optimal solution $\bsb{\hat{B}}$ with $\bsb{\hat{B}}=\frac{1}{1+\eta}\bsb{A}_I$. It follows that $f_0(\bsb{\hat{B}})=\frac{1}{2}\|\bsb{A}\|_F^2-\frac{1}{2}\sum_{i,j\in I}a_{ij}^2$. Therefore, the quantile thresholding $\Theta^\#(\bsb{A};m,\eta)$ yields a global minimizer.
\end{proof}

Using Lemma~\ref{lemmaQuantile}, we can prove the function value decreasing property; the remaining part follows similar lines of Theorem~\ref{th:fValDecrease} because of the separability  of $F$.

\section{Proof of Theorem~\ref{th:stability} } %and Corollary~\ref{cor:stability}
\label{app:stability}
%\textbf{Existence.}
Let $\bsbf(\bsbx) = \bsbL \bsbpi( \bsbA  \bsbx + \bsbu) - \bsbD \bsbx + \bsbc$, where $\bsbx$ is short for $\bsbx(t)$.
First, we  prove the existence of an equilibrium. It suffices  to show that there is a solution to
$\bsbf(\bsbx) = \bsb{0}$ or $\bsbx = \bsbD^{-1} \bsbL \bsbpi( \bsbA  \bsbx + \bsbu) + \bsbD^{-1} \bsbc=:\varphi(\bsbx)$. Obviously,  the mapping $\varphi$ is continuous and bounded (say $\| \varphi \|_{\infty} \leq M$), Brouwer's fixed point theorem \cite{brouwer1911abbildung} indicates the existence of at least one  equilibrium in $[-M, M]^n$.

%\textbf{Stability.}
Let $\bsbx^*$ be an equilibrium point, i.e., $\bsbf(\bsbx^*)=\bsb{0}$. Construct a Lyapunov function candidate  $V(\bsbx) = \frac{1}{2} (\bsbx - \bsbx^*)^{\tran} \bsbP (\bsbx - \bsbx^*)$ with $\bsbP$ positive definite and  to be determined.
Then
\begin{align*}
&\frac{\rd V(\bsbx)}{\rd t}= V'(\bsbx) \bsbf(\bsbx) \\
=& (\bsbx - \bsbx^*)^{\tran} \bsbP\bsbf(\bsbx)\\ =& (\bsbx - \bsbx^*)^{\tran} \bsbP (\bsbf(\bsbx)-\bsbf(\bsbx^*))\\
=& -(\bsbx - \bsbx^*)^{\tran} \bsbP\bsbD (\bsbx- \bsbx^*)  + (\bsbx - \bsbx^*)^{\tran} \bsbP\bsbL (\bsbpi( \bsbA  \bsbx + \bsbu) - \bsbpi( \bsbA  \bsbx^* + \bsbu))  \\
=& -(\bsbx - \bsbx^*)^{\tran} (\bsbP\bsbD - \bsbP\bsbL \bsbG \bsbA) (\bsbx- \bsbx^*)\\
%& = -(\bsbx - \bsbx^*)^{\tran} (\bsbD - \bsbL \bsbG \bsbA) (\bsbx- \bsbx^*)/2 - (\bsbx - \bsbx^*)^{\tran} (\bsbD - \bsbA^{\tran} \bsbG  \bsbL) (\bsbx- \bsbx^*) / 2\\
=&-(\bsbx - \bsbx^*)^{\tran} \left(\frac{\bsbP\bsbD+\bsbD\bsbP^{\tran}}{2} - \frac{\bsbP\bsbL \bsbG \bsbA + \bsbA^{\tran} \bsbG  \bsbL \bsbP^{\tran}}{2}\right) (\bsbx- \bsbx^*),
\end{align*}
where
$$\bsbG = \mbox{diag}\left\{\frac{\pi( \bsba_i^{\tran}  \bsbx  + \bsbu)-\pi( \bsba_i^{\tran}  \bsbx^* + \bsbu)}{ \bsba_i^{\tran} x_i -  \bsba_i^{\tran} x_i^*}\right\}.$$
 It is easy to verify that $\bsbG = \mbox{diag}\{\bsbpi'(\bsbxi)\}\preceq \bsbI/4$, and thus $\bsbP\bsbL \bsbG \bsbA + \bsbA^{\tran} \bsbG  \bsbL \bsbP^{\tran}\preceq (\bsbP\bsbL   \bsbA + \bsbA^{\tran}  \bsbL\bsbP^{\tran})/4$.
It is well known \cite{lasalle1976stability} that under \eqref{psdcond0}, the   Lyapunov equation
\begin{align}
\bsbP \left(\bsbD - \frac{\bsbL   \bsbA}{4}\right) +  \left(\bsbD - \frac{\bsbL   \bsbA}{4}\right)^{\tran} \bsbP^{\tran} = -\bsbR \label{lyapeq}
\end{align}
is solvable  and uniquely determines a positive definite $\bsbP$ for any  positive definite $\bsbR$.
Therefore, $V$ is indeed a Lyapunov function for the nonlinear dynamical system \eqref{eq:sigode1_recall}.
Moreover,  \eqref{psdcond0} implies
$$\frac{\rd V(\bsbx)}{\rd t} \leq -\varepsilon_0 \|\bsbx - \bsbx^*\|_2^2 \leq -\varepsilon V(\bsbx)$$
for some $\varepsilon_0, \varepsilon > 0$. By the Lyapunov stability theory---see, e.g., \cite{haddad2008nonlinear} (Chapter 3), \eqref{eq:sigode1_recall} must be  globally exponentially stable. The uniqueness of the equilibrium is implied by the global exponential stability.

The second result can be shown by setting $\bsbP=\bsbI$ in \eqref{lyapeq}; details are omitted.

\section{Proof of Theorem~\ref{th:s3conv} }
\label{app:s3conv}

Based on the proof of Theorem \ref{th:fValDecrease}, the modified Step 2 does not affect the convergence of function value  because at each iteration a global optimum  of \eqref{objstep2} is obtained.
It remains to show that $\tilde \bsbB^{(j)}$  generated by the modified Step 4 improves $\tilde \bsbB^{(j-1)}$ in terms of reducing the objective function value, and $\bsbB^{(j)}$ obeys the stability condition.

Let $\tilde{\bsb{\Xi}}=\tilde \bsbB^{(j-1)} - \tilde\bsbX^{\tran} \bsbW \bsbxi(\tilde\bsbX \tilde\bsbB^{(j-1)}, \tilde\bsbY) (\bsbK^{(j)})^{-}$, $\bsb{\Lambda}^{(j)} = \bsb{\Lambda} \cdot(\bsbK^{(j)})^{-}$. (The modified Step 2 may result in zeros in $\tilde{\bsb{\Xi}}$ to make the associated  activation terms  in $\bsb{\pi}(\tilde\bsbX \tilde\bsbB)$ vanish; using the pseudoinverse can handle the issue and maintain the decreasing property.) Based on the argument of Appendix \ref{app:fValDecrease},  the problem in Step 4 reduces to
\begin{align}
\begin{split}\min_{\tilde\bsbB=[\bsbu, \bsbB^\tran]^\tran} \frac{ 1 }{2 }\|\tilde\bsbB-\tilde{\bsb{\Xi}}\|_F^2 + \| \bsb{\Lambda}^{(j)} \circ \bsbB \|_1, \\ \mbox{ s.t. } (\bsbL^{(j)}   \bsbB^\tran + \bsbB  \bsbL^{(j)})/2   \preceq  4 \bsbD^{(j)}.
\end{split}
\end{align}

Therefore it suffices to prove Lemma \ref{iterBopt} for any given ${\bsb{\Xi}}(=\tilde{\bsb{\Xi}}[2:\mbox{end}, :])$, $\bsb{\Lambda}^{(j)}$, $\bsbL^{(j)}\succeq\bsb{0}$, $\bsbD^{(j)}\succeq\bsb{0}$, and  $\bsbB^{(j-1)}$.
(In fact, the lemma holds given any  initialization  of  $\bsbC_3$.)
\begin{lemma} \label{iterBopt}
For any $j\geq 1$,
, $\frac{ 1 }{2 }\|\bsbB^{(j)}-{\bsb{\Xi}}\|_F^2 + \| \bsb{\Lambda}^{(j)} \circ \bsbB^{(j)} \|_1 \leq \frac{ 1 }{2 }\|\bsbB^{(j-1)}-{\bsb{\Xi}}\|_F^2 + \| \bsb{\Lambda}^{(j)} \circ \bsbB^{(j-1)} \|_1$, and $(\bsbL^{(j)}   \bsbB^{(j)\tran} + \bsbB^{(j)}  \bsbL^{(j)})/2   \preceq  4 \bsbD^{(j)}$.
\end{lemma}
\begin{proof}
Define $f(\bsbB)=\frac{ 1 }{2 }\|\bsbB-{\bsb{\Xi}}\|_F^2 + \| \bsb{\Lambda} \circ \bsbB \|_1$, and $g(\bsbB', \bsbC', \bsbB, \bsbC) = \frac{ 1 }{2 }\|\bsbB'-{\bsb{\Xi}}\|_F^2 + \| \bsb{\Lambda} \circ \bsbB \|_1 +\frac{1}{2} \| \bsbB - \bsbB'\|_F^2 +\frac{1}{2} \| \bsbC - \bsbC'\|_F^2+\langle \bsbB'-{\bsb{\Xi}}, \bsbB - \bsbB'\rangle$. Then $f(\bsbB') = g(\bsbB', \bsbC', \bsbB', \bsbC')$ and $g(\bsbB', \bsbC', \bsbB, \bsbC) \geq f(\bsbB)$.

On the other hand, given $(\bsbB', \bsbC')$, we can write $g$ as a function of $(\bsbB, \bsbC)$:  $\frac{1}{2} \| [\bsb{\Xi}, \bsb{C}'] - [\bsbB, \bsbC] \|_F^2+ \| \bsb{\Lambda} \circ \bsbB \|_1$ (up to additive functions of $\bsbB'$ and $\bsbC'$). Based on Lemma \ref{globalopt}, with $\bsb{\Xi}_{\bsbC}=\bsbC'$, $g(\bsbB', \bsbC', \bsbB', \bsbC') \geq g(\bsbB', \bsbC', \bsbB^o, \bsbC^o)\geq f(\bsbB^o)$ and $ (\bsbL {\bsbB^o}^\tran + \bsbB^o \bsbL) /2= \bsbC^o \preceq 4 \bsbD$.

Applying the result to the modified Step 4 in Section \ref{sec:s3} yields the desired result.
\end{proof}
 \begin{lemma}\label{globalopt}
Consider the sequence of $(\bsbB_3, \bsbC_3)$ generated by the following procedure, with the operators $\mathcal P^1$, $\mathcal P^2$, $\mathcal P^3$ defined in Lemma \ref{proj_l1}, Lemma \ref{proj_lin} and Lemma \ref{proj_spec}, respectively:
\begin{algorithmic}
\STATE \textit{0) } $\bsbB_3 \gets \bsb{\Xi}_{\bsbB}$, $\bsbC_3 \gets \bsb{\Xi}_{\bsbC}$, $\bsbP = \bsb{0}$, $\bsbQ_{\bsbB}=\bsb{0}$, $\bsbQ_{\bsbC}=\bsb{0}$, $\bsbR=\bsb{0}$
%\WHILE{not converged}
\REPEAT
\STATE \textit{1) } $\bsbB_1 \gets {\mathcal P}^1 ( \bsbB_3 + \bsbP; \bsb{\Lambda})$, $\bsbC_1 \gets \bsbC_3$, $\bsbP \gets \bsbP + \bsbB_3 - \bsbB_1$
\STATE \textit{2) }  $[\bsbB_2, \bsbC_2] \gets {\mathcal P}^2(\bsbB_1 + \bsbQ_{\bsbB}, \bsbC_1 + \bsbQ_{\bsbC}; \bsbL) $, $[\bsbQ_{\bsbB}, \bsbQ_{\bsbC}] \gets [\bsbQ_{\bsbB}, \bsbQ_{\bsbC}] + [\bsbB_1, \bsbC_1] - [\bsbB_2, \bsbC_2]$ % $\bsbQ_{\bsbC} \gets \bsbQ_{\bsbC} + \bsbC_1 - \bsbC_2$
\STATE \textit{3) } $\bsbB_3 \gets \bsbB_2$, $\bsbC_3 \gets {\mathcal P}^3(\bsbC_2 + \bsbR; \bsbD)$, $\bsbR\gets \bsbR + \bsbC_2 - \bsbC_3$.
\UNTIL{convergence }
\end{algorithmic}
Then, the sequence of iterates converges to a globally optimal solution $(\bsbB^o, \bsbC^o)$ to
\begin{align}
\begin{split}\min_{\bsbB, \bsbC} \frac{1}{2} \|[\bsb{\Xi}_{\bsbB}, \bsb{\Xi}_{\bsbC}] - [\bsbB, \bsbC] \|_F^2 + \| \bsb{\Lambda} \circ \bsbB\|_1 \\\mbox{ s.t. } \bsbC = (\bsbL \bsbB^\tran + \bsbB \bsbL) /2, \bsbC \preceq 4 \bsbD.
\end{split}
\end{align}
\end{lemma}
\begin{proof}
With  the following three lemmas,  applying Theorem 3.2 and Theorem 3.3 in \cite{Bauschke2008} yields the strict convergence of the iterates  and the global optimality of the limit point. \end{proof}
\begin{lemma} \label{proj_l1}
Let  $\mathcal P^1(\bsb{\Phi})$ be the optimal solution to
 \begin{equation}
\min_{\bsbB} \frac{1}{2} \|\bsbB-\bsb{\Phi}\|_F^2 + \| \bsb{\Lambda} \circ \bsbB\|_1.
\end{equation}
 Then  $\mathcal P^1(\bsb{\Phi}; \bsb{\Lambda}) = \Theta_S(\bsb{\Phi}; \bsb{\Lambda})$ where $\Theta_S$, applied componentwise on $\bsb{\Phi}$, is the soft-thresholding rule given in Section \ref{sec:univariate}. \end{lemma}
\begin{proof} Apply Lemma 1 in \cite{SheGLMTISP}.
\end{proof}
\begin{lemma} \label{proj_lin}
The optimal solution to
 \begin{equation}
\min_{\bsbB, \bsbC} \frac{1}{2} \|[\bsbB \ \bsbC]-[\bsb{\Phi}_{\bsbB} \ \bsb{\Phi}_{\bsbC}]\|_F^2 \mbox{ s.t. } \bsbC = (\bsbL \bsbB^\tran + \bsbB \bsbL) /2.
\end{equation}
 is given by $\mathcal P^2(\bsb{\Phi}_{\bsbB}, \bsb{\Phi}_{\bsbC}; \bsbL)=[\bsbB^o, \bsbC^o]$ with $\bsbB^o=[b_{i,j}^o]$,  $b_{i,j}^o=\psi_{i,j} \frac{2+l_i^2}{2+l_i^2 + l_j^2} - \psi_{j,i} \frac{l_i l_j }{2+l_i^2 + l_j^2}$, and $\bsbC^o=(\bsbL {\bsbB^{o}}^{\tran} + {\bsbB^o}^{\tran} \bsbL) /2$, where $\bsb{\Psi}=[\psi_{i,j}]=\bsb{\Phi}_{\bsbB} + (\bsb{\Phi}_{\bsbC} + \bsb{\Phi}_{\bsbC}^\tran)\bsbL/2$.
\end{lemma}
\begin{proof}
Let $f(\bsbB) =  \|\bsbB- \bsb{\Phi}_{\bsbB}\|_F^2/2 + \|  (\bsbL \bsbB^\tran + \bsbB \bsbL) /2-\bsb{\Phi}_{\bsbC}\|_F^2/2$. It is not difficult to obtain the gradient (details omitted): $\bsbB - \bsb{\Phi}_{\bsbB} + (\bsbL \bsbB^\tran \bsbL + \bsbB \bsbL^2)/2 - (\bsb{\Phi}_{\bsbC} + \bsb{\Phi}_{\bsbC}^\tran)\bsbL/2$. The optimal $\bsbB$ and $\bsbC$ can be evaluated accordingly.
\end{proof}

\begin{lemma} \label{proj_spec}
Let $\mathcal P^3(\bsb{\Phi}; \bsbD)$ be the optimal solution to
 \begin{equation}
\min_{\bsbC} \frac{1}{2} \|\bsbC-\bsb{\Phi}\|_F^2 \mbox{ s.t. } \bsbC \mbox{ is symmetric and satisfies }\bsbC \preceq 4\bsb{D}. \label{optCOnly}
\end{equation}
 Then it is given by $\bsb{U}\mbox{diag}\{\min(s_i,0)\}\bsb{U}^{\tran}+4\bsb{D}$, where $\bsb{S}=\mbox{diag}\{s_1, \cdots, s_n\}$ and $\bsbU$ are  from  the spectral decomposition of    $(\bsb{\Phi}+\bsb{\Phi}^\tran)/{2}-4\bsb{D}=\bsb{U}\bsb{S}\bsb{U}^{\tran}$.\end{lemma}
\begin{proof}
Because  $\bsbC$ is symmetric (but   $\bsb{\Phi}$ may not be), we have
\begin{equation}\begin{aligned}
&\|\bsbC-\bsb{\Phi}\|_F^2 \\
= &\sum_{1\leq i<j \leq p}[(c_{ij}-\phi_{ij})^2+(c_{ij}-\phi_{ji})^2] + \sum_{i=1}^p(c_{ii}-\phi_{ii})^2
\\ = & \sum_{1\leq i<j \leq p}2(c_{ij}-\frac{\phi_{ij}+\phi_{ji}}{2})^2 + \sum_{i=1}^p(c_{ii}-\phi_{ii})^2 + const(\bsb{\Phi}) \\
=& \,\|\bsbC-\frac{\bsb{\Phi}+\bsb{\Phi}^\tran}{2}\|_F^2 + const(\bsb{\Phi}),
\end{aligned}\end{equation}
where $const(\bsb{\Phi})$ is a term that does not depend on $\bsbC$.
Therefore, problem \eqref{optCOnly} is equivalent to
 \begin{equation}\begin{gathered}
\min_{\bsbC} \frac{1}{2} \|\bsbC-\frac{\bsb{\Phi}+\bsb{\Phi}^\tran}{2}\|_F^2,
\mbox{s.t. } \bsbC-4\bsb{D} \preceq \bsb{0}.
\end{gathered}\end{equation}
The optimality of $\mathcal P^3(\bsb{\Phi}; \bsbD)$ can then be argued by  von Neumann's trace inequality \cite{vonNeumann1937,SheTISPMat}. % \cite{vonNeumann1937}.
\end{proof}

%\begin{center}\sc{Acknowledgement }\end{center}
%The authors are grateful to the associate editor and the two anonymous
%referees for their careful comments and useful suggestions.

%\bibliographystyle{IEEEtran}
\bibliographystyle{IEEEtranN}
\bibliography{sigmoid}
%\addbibresource{sigmoid}

\end{document}